\documentclass[11pt]{article}
\usepackage[dvipsnames]{xcolor}
\usepackage{ulem}
\usepackage{pgfplots}
\usepackage{booktabs}
\usepackage{graphicx,wrapfig,lipsum}
\usepackage{float}    % For tables and other floats
\usepackage{verbatim} % For comments and other
\usepackage{amsmath, amssymb, amsthm, mathtools}  % For math
\usepackage{subfig}   % For subfigures
\usepackage[colorlinks=true,citecolor=blue,linkcolor=blue,urlcolor=blue]{hyperref}
\usepackage{cleveref}
\usepackage{bookmark}
\usepackage{fullpage}
\usepackage{enumerate}
\usepackage{paralist}
\usepackage{xspace}
\usepackage{multirow}
\usepackage{caption}
\usepackage{bbm}
\usepackage{bm}
\usepackage{url}
\usepackage{lipsum}
\usepackage{algorithm}
\usepackage{algorithmic}
\usepackage{color}
\usepackage{microtype}
\usepackage[section]{placeins}
\usepackage{lscape}
\usepackage{multirow}
\usepackage{todonotes}

\newcommand{\vertiii}[1]{{\left\vert\kern-0.25ex\left\vert\kern-0.25ex\left\vert #1
		\right\vert\kern-0.25ex\right\vert\kern-0.25ex\right\vert}}

\makeatletter

\newcommand{\sign}{\mathrm{sgn}}

\DeclareMathOperator{\poly}{\textsf{poly}}

\newcommand{\err}{\epsilon}

\newcommand{\rank}{\mathrm{rank}}

\newcommand{\inner}[2]{\left\langle #1, #2 \right\rangle}

\newcommand{\norm}[1]{\left\lVert#1\right\rVert}

{\medskip}

{\hspace*{\fill}$\Box$\par\vspace{4mm}}
{\hspace*{\fill}$\Box$\par}

\DeclareMathOperator*{\argmin}{arg\,min}

\newcommand{\cA}{\mathcal{A}}
\newcommand{\cB}{\mathcal{B}}
\newcommand{\cC}{\mathcal{C}}

\newcommand{\cG}{\mathcal{G}}

\newcommand{\cI}{\mathcal{I}}
\newcommand{\cJ}{\mathcal{J}}

\newcommand{\cM}{\mathcal{M}}
\newcommand{\cN}{\mathcal{N}}
\newcommand{\cO}{\mathcal{O}}

\newcommand{\cS}{\mathcal{S}}
\newcommand{\cT}{\mathcal{T}}
\newcommand{\cU}{\mathcal{U}}
\newcommand{\cV}{\mathcal{V}}
\newcommand{\cW}{\mathcal{W}}

\newcommand{\cZ}{\mathcal{Z}}

\newcommand{\bE}{\mathbb{E}}

\newcommand{\bP}{\mathbb{P}}

\newcommand{\bR}{\mathbb{R}}

\newcommand{\fW}{\mathbf{W}}

\newcommand{\trueX}{X^\star}
\newcommand{\trueW}{W^\star}

\newcommand{\fc}{f^{\mathrm{c}}}

\newcommand{\vecc}{\mathrm{vec}}

\newtheorem{theorem}{Theorem}
\newtheorem{definition}{Definition}
\newtheorem{lemma}{Lemma}

\newtheorem{remark}{Remark}
\newtheorem{corollary}{Corollary}
\newtheorem{proposition}{Proposition}

\newtheorem{assumption}{Assumption}

\begin{document}

	\title{Can Learning Be Explained By Local Optimality In Robust Low-rank Matrix Recovery?}
	\author{Jianhao Ma$^*$ and Salar Fattahi$^+$\vspace{2mm}\\
		Department of Industrial and Operations Engineering\\
		University of Michigan\\
		$^*$\href{mailto:jianhao@umich.edu}{jianhao@umich.edu}, $^+$\href{mailto:fattahi@umich.edu}{fattahi@umich.edu}
	}
	
	\maketitle
	\begin{abstract}
		 We explore the local landscape of low-rank matrix recovery, focusing on reconstructing a $d_1\times d_2$ matrix $X^\star$ with rank $r$ from $m$ linear measurements, some potentially noisy. When the noise is distributed according to an outlier model, minimizing a nonsmooth $\ell_1$-loss with a simple sub-gradient method can often perfectly recover the ground truth matrix $X^\star$. Given this, a natural question is what optimization property (if any) enables such learning behavior. The most plausible answer is that the ground truth $\trueX$ manifests as a \textit{local optimum} of the loss function. In this paper, we provide a strong negative answer to this question, showing that, under moderate assumptions, the true solutions corresponding to $X^\star$ {\it do not} emerge as local optima, but rather as strict saddle points---critical points with strictly negative curvature in at least one direction. Our findings challenge the conventional belief that all strict saddle points are undesirable and should be avoided. 
	\end{abstract}

	\section{Introduction}

In machine learning (ML), learning the true model—or one with strong generalization properties—often requires minimizing a nonconvex loss function, typically using first-order methods like the (sub-)gradient method or its stochastic and accelerated variants. These methods have demonstrated remarkable success in learning the true model across a range of modern ML tasks, from low-rank matrix recovery to training deep neural networks; see recent surveys~\cite{chen2018harnessing, chi2019nonconvex, zhang2020symmetry, sun2020optimization, sun2019optimization}. Given this success, a natural question thus arises: \textit{What enables these algorithms to effectively learn the true model}?

It is well known that, under reasonable assumptions, first-order algorithms are guaranteed to converge to approximate first-order stationary points. For general nonconvex functions, these points can be either local minima or (strict) saddle points.
Unlike local minima, saddle points are widely regarded as undesirable solutions that should be avoided in practice. For example, \cite{lee2016gradient} state that “saddle points have long been considered a significant obstacle in continuous optimization.” Similarly, \cite{dauphin2014identifying} argue that the presence of saddle points poses a “fundamental impediment to rapid high-dimensional nonconvex optimization.” This perspective is further complemented by recent findings showing that, for \textit{generic} functions, first-order methods escape strict saddle points and converge exclusively to local minima~\cite{davis2022proximal, lee2016gradient, bianchi2023stochastic}.
Given this, it is reasonable to speculate that the learning capability of local-search algorithms stems from the true solution emerging as a \textit{local minimum} of the loss function, thereby providing an affirmative answer to the question posed in the title of this paper.

In this paper, we examine the extent to which this conjecture holds for an important class of problems in ML: \textit{robust low-rank matrix recovery} with nonsmooth $\ell_1$-loss. In this problem, the goal is to recover a low-rank matrix $X^\star$ from a limited number of linear measurements, some of which are heavily corrupted by noise. This problem lies at the heart of several important applications, including motion detection~\cite{bouwmans2014robust}, collaborative filtering~\cite{luo2014efficient}, state estimation~\cite{molybog2020role, zhang2019spurious}, and, more recently, teaching arithmetic to transformers~\cite{lee2023teaching}. It also serves as an ideal test-bed to investigate the aforementioned conjecture, where both the model and data admit a crisp mathematical formulation. Our key finding is that, under moderate assumptions on the model and sample size, the true solutions corresponding to $X^\star$ {\it do not} emerge as local minima but as \textit{strict saddle points}---critical points with strictly negative curvature in at least one direction.

Our result has two surprising implications. First, it challenges the conventional belief that all strict saddle points are undesirable and should be avoided. In fact, we demonstrate the opposite: certain strict saddle points in robust low-rank matrix recovery can lead to better generalization guarantees than both local and global solutions.
Second, we show that, unlike its smooth counterparts, not all strict saddle points are easy to escape in the robust low-rank matrix recovery. This suggests that the problem may not be ``generic enough'', in the sense defined by existing theory (see, e.g., ~\cite[Theorem 2.9]{davis2022proximal}), since a simple sub-gradient method with a small and/or random initial point does converge to the ground truth of the robust low-rank matrix recovery~\cite{ma2023global}, even though it is a strict saddle point.

\subsection{Robust Low-rank Matrix Recovery}
We study the optimization landscape of \textit{low-rank matrix recovery}, where the goal is to recover a matrix $X^\star\in \bR^{d_1\times d_2}$ with rank $r\leq \min\{d_1,d_2\}$ from $m$ linear and possibly noisy measurements $y = \cA(\trueX)+\epsilon$.
	The measurement operator $\cA: \bR^{d_1\times d_2}\to \bR^m$ is defined as $\cA(X) = [\langle A_1, X\rangle\ \ \dots \ \ \langle A_m, X\rangle]^\top$, and $\epsilon\in \bR^m$ is the additive noise vector. 

This problem is commonly solved via a factorization technique called \textit{Burer-Monteiro method} (BM)~\cite{burer2003nonlinear, burer2005local, chi2019nonconvex}. In BM, the target low-rank matrix is modeled as $\trueX = \trueW_1\trueW_2$, where $\trueW_1\in \bR^{d_1\times k}$ and $\trueW_2\in \bR^{k\times d_2}$ are unknown factors each with search rank $k\geq r$. Based on this definition, any pair of factors $(\trueW_1,\trueW_2)$ that satisfies $\trueX = \trueW_1\trueW_2$ is called a {\it true solution}. 
When the true rank $r$ is known and $k=r$, the problem is called \textit{exactly-parameterized}. The more practical setting where the true rank $r$ is unknown and instead over-estimated with $k>r$ is referred to as \textit{over-parameterized}.
 Given this factorized model, the problem of recovering the true solution is typically formulated as follows: 
 \begin{align}\tag{\texttt{BM}}\label{BM}
		\min_{\fW = (W_1,W_2)} f_{\ell_q}(\fW) := \frac{1}{m}\norm{y-\cA(W_1W_2)}^q_{\ell_q}.
	\end{align}

 However, this strategy presents two primary challenges. First, owing to its nonconvex nature, local-search algorithms may converge to sub-optimal solutions, falling short of achieving global optimality. Second, even if an optimal solution can be obtained efficiently, it may not coincide with the true solution $\trueX$. This fact is supported by a well-known statistical trade-off: in the presence of an unstructured measurement noise $\epsilon \neq 0$, the best achievable result in low-rank matrix recovery is to approximate $X^\star$ up to a minimax statistical error \cite{candes2011tight}. 
Despite the aforementioned statistical barrier, if the measurement noise $\epsilon$ follows an outlier model, minimizing the nonsmooth $\ell_1$-loss $f_{\ell_1}$ (i.e., by setting $q=1$ in~\ref{BM}) can often recover $X^\star$ \textit{perfectly}, that is, with no statistical error~\cite{li2020nonconvex, ma2023global, ma2021sign, charisopoulos2021low, tong2021low}.

The unique ability of the $\ell_1$-loss to exactly recover the true solution $X^\star$ is particularly appealing, as it exemplifies the ``learning'' behavior observed in modern machine learning models. An important open question is what enables such a learning behavior. Is it due to some form of local optimality of the true solutions?

\subsection{Summary of Contributions}
This paper presents a strong negative answer to the aforementioned question. To ensure that our results are statistically meaningful, we consider two popular statistical models for robust low-rank matrix recovery: \textit{matrix completion} and \textit{matrix sensing}. In matrix completion, the measurement operator $\cA$ is assumed to be a random element-wise projection (see Definition~\ref{asp_completion}), while in matrix sensing, the measurement matrices $\{A_i\}_{i=1}^m$ are assumed to be Gaussian (see Definition~\ref{asp_sensing}). 

At a high level, our contribution can be summarized as follows:
When $\cA$ consists of Gaussian measurements (as in matrix sensing), the true solutions \textit{do not} manifest as local or global optima unless there are \textit{sufficient} measurements—specifically, on the order of $m \gtrsim \max\{d_1, d_2\}k$. Therefore, if the model is highly over-parameterized, as many as $m \gtrsim d_1d_2$ measurements may be required to ensure that the true solution $\trueX$ emerges as a local or global optimum. On the positive side, the true solutions do manifest as critical and strict saddle points, provided the measurement size exceeds the information-theoretically optimal threshold of $m \gtrsim \max\{d_1, d_2\}r$. In contrast, when $\cA$ is an element-wise projection (as in matrix completion), the true solution $\trueX$ does not emerge as a local or global optimum, \textit{regardless} of the number of measurements.

In what follows, we provide a more detailed summary of our results for {matrix sensing} and {matrix completion}, addressing their symmetric and asymmetric versions separately.

 \paragraph{Symmetric matrix sensing.}
 We show that, for any true solution to manifest as a local optimum of $f_{\ell_1}$ for the symmetric matrix sensing, the sample size must inevitably scale with the search rank. 
 {In particular, we show that the landscape of the symmetric matrix sensing with $\ell_1$-loss undergoes two sharp transitions: when $m\lesssim dr$, the exact recovery of the true solution is impossible due to the existing information-theoretical barriers~\cite{candes2011tight}. When $dr\lesssim m\lesssim dk$, all true solutions become {\it strict saddle points}, that is, they all become critical with strictly negative curvature (see Definition~\ref{def_ssp}). 
 Subsequently, as the measurement size extends to satisfy $m\gtrsim dk$, a result by~\cite{ding2021rank} shows that all true solutions emerge as the unequivocal global optima of \ref{BM}.}

 \begin{table}[]
     \centering\footnotesize
     \begin{tabular}{c||c|c|c}
          Sample size  & $\frac{m}{\max\{d_1,d_2\}}\lesssim r$ & $r\lesssim\frac{m}{\max\{d_1,d_2\}}\lesssim k$ & $k\lesssim\frac{m}{\max\{d_1,d_2\}}$ \\
          \hline\hline
          $\rank(\trueW_1)+\rank(\trueW_2)\leq {c_1}r$  & \multirow{3}{*}{\shortstack{Recovery is\\ impossible~\cite{candes2011tight}}} & \shortstack{Critical and strict saddle} & \multirow{3}{*}{\shortstack{Global minimum}}\\
          ${c_1}r< \rank(\trueW_1)+\rank(\trueW_2)\leq {c_2k}$  & & Non-critical or strict saddle & \\
          ${c_2k}< \rank(\trueW_1)+\rank(\trueW_2)\leq 2k$  & & Non-critical & \\
     \end{tabular}
     \caption{\footnotesize We show that the local landscape of asymmetric matrix sensing with $\ell_1$-loss varies depending on the sample size and the factorized rank of the true solution. {In this table, $c_1,c_2$ are universal constants that satisfy $2\leq c_1$ and $0<c_2<2$.}}
     \label{tab_ams}
 \end{table}

\paragraph{Asymmetric matrix sensing.} {We extend our results to the asymmetric matrix sensing, proving  
that the local landscape around any true solution $(\trueW_1, \trueW_2)$ that satisfies $\trueW_1\trueW_2 = \trueX$ varies depending on the sample size and their \textit{factorized rank} defined as $\rank(\trueW_1)+\rank(\trueW_2)$. Similar to the symmetric case, the exact recovery is impossible when $m\lesssim \max\{d_1,d_2\}r$. However, unlike the symmetric matrix sensing, in the regime where $\max\{d_1,d_2\}r\lesssim m\lesssim \max\{d_1,d_2\}k$, not every true solution emerges as a strict saddle point. In fact, we prove that those with a high factorized rank, are not critical points. In contrast, those with a small factorized rank emerge as strict saddle points. {Similar to the symmetric case,} once the number of samples exceeds $m \gtrsim \max\{d_1,d_2\}k$, all true solutions emerge as global optima. Table~\ref{tab_ams} provides a summary of our results for this problem. }

\paragraph{Symmetric matrix completion.} For the symmetric matrix completion, we show that the true solutions are either non-critical or strict saddle points of~\ref{BM}, {and thus are neither global nor local optima,} so long as at least one of the observed diagonal entries of $\trueX$ is corrupted with a positive noise. This result holds even when a slight overestimation of the true rank occurs  (for instance, $k= r+1$), or when the full matrix $\trueX$ is observed. {This result highlights an important distinction: unlike in symmetric matrix sensing, introducing additional samples does not transform the true solutions into global or local optima.}

\paragraph{Asymmetric matrix completion.} Finally, we extend our results to the asymmetric matrix completion. In particular, we show that the true solutions with a factorized rank of at most $k-1$ are either non-critical or strict saddle points. {Therefore, they are neither global nor local optima.} This result holds with any arbitrarily small (but constant) fraction of noise. {Moreover, we prove that a subset of the true solutions are non-critical points, even when the overestimation of the true rank is mild (for instance, $k= r+1$), or when the full matrix $\trueX$ is observed.} Finally, we study the landscape of asymmetric matrix completion for matrices exhibiting a property known as \textit{coherence}. These are matrices where at least one of their left or right singular vectors aligns with a standard unit vector.  We show that the asymmetric matrix completion with a coherent ground-truth $\trueX$ has the most undesirable landscape, as none of the true solutions are even critical points. This result holds true even when the complete matrix $\trueX$ is observed, or strikingly, when the model is exactly parameterized.

\subsection{Proof Idea: Parametric Perturbation Sets}
Recall that a point ${\fW}^\star = ({W}^\star_1, {W}^\star_2)$ is called a {true solution} if it satisfies ${W}^\star_1 {W}^\star_2 = \trueX$. Accordingly, let $\mathcal{W}$ denote the set of all true solutions $\mathcal{W} = \{(W_1, W_2): W_1 W_2 = X^\star\}$. The set of true solutions for the symmetric case is similarly defined as $\mathcal{W} = \{W: WW^\top = X^\star\}$. 
At a high level, we prove the ``inexistence of local optima in $\cW$'' by proving the following statement:
There exists $\gamma_0>0$ and $\alpha\geq 1$ such that for every $\gamma\leq \gamma_0$
{\begin{equation}
    \begin{aligned}
        &\bP_{\cA, \epsilon}\Bigl(\max_{\fW^\star\in \cW} \min_{\|\Delta\fW\|= 1}\left\{f_{\ell_1}(\fW^\star+\gamma\Delta\fW)-f_{\ell_1}(\fW^\star)\right\}\leq -\Omega(\gamma^\alpha)\Bigr)\geq 1-\exp\left(-\poly(d_1,d_2,m,r)\right).
    \end{aligned}
    \label{eq_main}
\end{equation}}
In the above inequality, the probability is defined over the possible randomness of the set of measurement matrices $\cA=\{A_1,\dots, A_m\}$ and the noise vector $\epsilon$. Moreover, $\poly(d_1,d_2,m,r)$ denotes a polynomial function of $d_1,d_2,m,r$. {Indeed, if the above statement holds with any $\alpha>0$, it implies that, with a high probability, none of the true solutions are local optima. Additionally, when the aforementioned inequality is satisfied with $\alpha=1$, it can be asserted with a high probability that none of the true solutions will manifest as critical points. Similarly, if the same inequality holds with $\alpha=2$, with a high probability, all true solutions are either non-critical or critical but strict saddle. {To distinguish between the non-critical and strict saddle points, we additionally need to establish similar matching lower bounds.}\footnote{For this argument to be valid, the choice of the direction $\Delta \fW$ must not depend on $\gamma$. In our initial presentation of the proof idea, we opted not to emphasize this subtle point to maintain clarity. However, we have ensured this independence in our technical results presented in the later sections.}}

{The primary challenge in proving the above statement lies in its \textit{uniform} nature: it must hold with high probability for the \textit{worst-case} choices of $\fW^\star\in \cW$ and ${\norm{\Delta\fW}_F= 1}$. 
To establish such a result, it is necessary to characterize the extremum behavior of the random process {$f_{\ell_1}(\fW^\star+\gamma\Delta\fW)-f_{\ell_1}(\fW^\star)$ over $\fW^\star\in \cW$ and ${\norm{\Delta\fW}_F= 1}$}.} One natural way to do so is by providing an explicit characterization of $\fW^\star\in \cW$ and {${\norm{\Delta\fW}_F= 1}$} that achieve this extremum behavior in terms of the random quantities $\cA,\epsilon$. Unfortunately, this characterization cannot be expressed explicitly. To circumvent this challenge, we introduce a \textit{parametric perturbation set} {$\cU(\fW^\star, \cA, \epsilon)\subset \{\Delta\fW: {\norm{\Delta\fW}_F= 1}\}$} that exhibits the following key properties: 
\begin{enumerate}
    \item With probability at least $1-\exp\left(-\poly(d_1,d_2,m,r)\right)$, the set {$\cU(\fW^\star, \cA, \epsilon)$} is non-empty for every $\fW^\star\in \cW$.
    \item {When $\cU(\fW^\star, \cA, \epsilon)$ is non-empty, we have $f_{\ell_1}(\fW^\star+\gamma\Delta\fW)-f_{\ell_1}(\fW^\star)\leq -\Omega(\gamma^\alpha)$ for every $\Delta\fW\in\cU(\fW^\star, \cA, \epsilon)$, $\fW^\star\in \cW$, and $\gamma\leq \gamma_0$.}
\end{enumerate}
 {Indeed, the above two statements together are sufficient to establish the correctness of~\eqref{eq_main}, since 
 \begin{equation*}
     \begin{aligned}
         \max_{\fW^\star\in \cW} \min_{{\norm{\Delta\fW}_F= 1}}\left\{f_{\ell_1}(\fW^\star+\gamma\Delta\fW)-f_{\ell_1}(\fW^\star)\right\} &\leq \max_{\fW^\star\in \cW} \min_{\Delta\fW\in\cU(\fW^\star, \cA, \epsilon)}\left\{f_{\ell_1}(\fW^\star+\gamma\Delta\fW)-f_{\ell_1}(\fW^\star)\right\}\\
         &\leq -\Omega(\gamma^\alpha),
     \end{aligned}
 \end{equation*}}
 with an overwhelming probability.
 Our key technical contribution resides in the explicit construction of such parametric perturbation sets that satisfy the aforementioned conditions.
	\subsection{Related Work}

 \paragraph{Effect of over-parameterization}
	For the noiseless symmetric matrix sensing and completion with exactly parameterized rank $k=r$,~\ref{BM} with $\ell_2$-loss satisfies the strict saddle property. This property implies that all critical points of the loss function, aside from those corresponding to true solutions, are strict saddle points~\cite{bhojanapalli2016global,ge2016matrix}. This result also holds for other variants of low-rank matrix recovery, such as phase retrieval~\cite{sun2018geometric}, as well as those with more general smooth loss functions~\cite{zhu2018global}. 
 In the asymmetric setting, earlier findings have indicated that augmenting~\ref{BM} with a balancing regularizer also maintains the strict saddle property~\cite{ge2017no, zhu2021global}. Notably, recent research has demonstrated that the inclusion of the balancing regularization is, in fact, unnecessary for preserving this property~\cite{li2020global}. In cases where the global strict saddle property may not hold globally, it can still be established locally around true solutions, with more lenient conditions~\cite{zhang2020many, ma2022sharp}.
Recently, these insights have been extended to the over-parameterized regime. Specifically,~\cite{zhang2022improved, zhang2021sharp} have shown that the \ref{BM} progressively reduces the occurrence of spurious local optima as the search rank $k$ exceeds the true rank $r$. However, this progress comes at the cost of increasing the sample size proportionally with the search rank $k$. To the best of our knowledge, no existing results establish the strict saddle property of~\ref{BM} with a sample complexity independent of the search rank.

 \paragraph{Effect of noise}
 Another line of work has studied the landscape of~\ref{BM} with $\ell_1$-loss, particularly in scenarios where a subset of the measurements are grossly corrupted with noise.  Recent results on matrix completion~\cite{fattahi2020exact} and matrix sensing~\cite{ma2021sign}, focusing on the case where $k=r=1$, have revealed that all critical points either exhibit small norm or closely approximate the true solution, even if up to a constant fraction of the measurements are corrupted with noise~\cite{fattahi2020exact, josz2022nonsmooth, ma2021sign}. Generalizing beyond the rank-1 case, only local results are known for the $\ell_1$-loss. In particular,~\cite{li2020nonconvex, charisopoulos2021low} have shown that, for the symmetric matrix sensing with exactly parameterized rank $k=r$ and $\ell_1$-loss, the global optima correspond to the true solutions and there are no spurious critical points locally around the true solutions, even if up to half of the measurements are corrupted with noise. These findings have recently been extended to the over-parameterized regime where $k\geq r$~\cite{ding2021rank}. However, analogous to the smooth case, the necessary sample size still scales with the search rank. To the best of our knowledge, these results have not been extended to asymmetric matrix sensing, symmetric matrix completion, or asymmetric matrix completion with a general rank configurations.

\paragraph{Convergence guarantees}
	Achieving global optimality for the true solutions typically necessitates a sample size that scales with the search rank. However, recent studies have revealed a promising contrast in the behavior of local-search algorithms by requiring significantly smaller sample sizes to converge to the true solutions. For instance, in noiseless matrix sensing with an $\ell_2$-loss, gradient descent has been shown to converge to the ground truth, provided that the sample size is proportional to the true rank, potentially being substantially smaller than the search rank~\cite{stoger2021small, xu2023power, zhang2022preconditioned, soltanolkotabi2023implicit}. This result has been recently extended to noisless matrix completion via a weakly-coupled leave-one-out analysis~\cite{ma2024convergence}. This result has been extended to matrix sensing with an $\ell_1$-loss, where it has been shown that the sub-gradient method, initialized with a small value, can converge to the ground truth when applied to~\ref{BM} with an over-parameterized rank~\cite{ma2023global}. This convergence occurs even in scenarios where the sample size does not scale with the search rank, and an arbitrarily large portion of the measurements is corrupted by noise. These findings strongly suggest the conjecture that the conditions for convergence of local-search algorithms may be considerably less stringent than the requirements for global optimality of the true solutions.

	\paragraph{Notations.}
	For a matrix $M$, its operator, Frobenius, and element-wise $\ell_q$ norms are denoted as $\norm{M}$, $\norm{M}_F$, and $\norm{M}_{\ell_q}$, respectively. {For a matrix $M\in \bR^{m\times n}$, its singular values are denoted as $\sigma_1(M)\geq\sigma_2(M)\geq\dots\geq \sigma_{\min\{m, n\}}(M):=\sigma_{\min}(M)$. Moreover, the minimum nonzero singular value of $M$ is denoted as $\sigma_{\min,+}(M)$.}
	The symbol $0_{m\times n}$ refers to $m\times n$ zero matrix. Similarly, we define $I_{m\times n} = [I_{m\times m}, 0_{m, n-m}]$ if $n>m$, and $I_{m\times n} = [I_{n\times n}, 0_{n, m-n}]^\top$ if $n<m$, where $I_{m\times m}$ and $I_{n\times n}$ are $m\times m$ and $n\times n$ identity matrices, respectively.
	For two matrices $X$ and $Y$ of the same size, their inner product is shown as $\inner{X}{Y}$. 
	The set $\cB_{\bar M}(\gamma)$ denotes the Euclidean ball with radius $\gamma$ centered at $\bar M\in \bR^{m\times n}$, i.e., $\cB_{\bar M}(\gamma) := \{M\in \bR^{m\times n}: \norm{M-\bar M}_{F}\leq \gamma\}$. We omit the subscript if the ball is centered at zero{, and we omit the radius parameter when $\gamma = 1$.} The set $\cO_{m,n}$ refers to the set of $m\times n$ orthonormal matrices. More precisely, $\cO_{m\times n}:=\{M\in\bR^{m\times n}: MM^\top = I_m\}$ if $m\leq n$, and $\cO_{m\times n}:=\{M\in\bR^{m\times n}: M^\top M = I_n\}$ if $n\leq m$. The Kronecker product of two matrices $M$ and $N$ is denoted by $M\otimes N$. For a matrix $M$, $\ker(M)$ and $\dim(\ker(M))$ denote the kernel of $M$ and the dimension of the kernel of $M$, respectively. We use $M[i, j]$ to denote the $(i, j)$-th element of $M$.
 
 Given an integer $n\geq 1$, we define $[n] = \{1,2,\dots, n\}$. Given a subset $\cI\in [n]$, we define $\cI^c = [n]\backslash \cI$. 	The sign function is defined as $\sign(x) = x/|x|$ if $x\not=0$, and $\sign(0) \in [-1,1]$.
	Given two sequences $f(n)$ and $g(n)$, the notation $f(n)\lesssim g(n)$ or $f(n)= O(g(n))$ implies that there exists a universal constant $C>0$ satisfying $f(n) \leq Cg(n)$. We also write $g(n)=\Omega(f(n))$ if $f(n)\lesssim g(n)$. Moreover, the notation $f(n)\asymp g(n)$ implies that $f(n)\lesssim g(n)$ and $g(n)\lesssim f(n)$. Throughout the paper, the symbols $C, c_1,c_2, \dots$ refer to universal constants whose precise value may change depending on the context.

\section{Main Results}
 \subsection{Problem Formulation}
	In this work, we study the landscape of~\ref{BM} with $\ell_1$-loss for two classes of low-rank matrix recovery, namely matrix sensing and matrix completion. 
	For matrix sensing, the measurement matrices are designed according to the following model.
	\begin{assumption}[Matrix Sensing Model]\label{asp_sensing}
		For every $1\leq i\leq m$, the entries of the measurement matrix $A_i$ are independently drawn from a standard Gaussian distribution with zero mean and unit variance.\footnote{More generally, matrix sensing is defined as a class of low-rank matrix recovery problems that satisfy the so-called \textit{restricted isometry property} (RIP). It is well-known that Gaussian measurements satisfy RIP with high probability. In fact, we expect that our results can be extended to measurements that satisfy RIP. However, we omit such an extension in this paper since it does not have any direct implications on our findings.}
	\end{assumption}
	Recall that in~\ref{BM}, the true solution is modeled as $W_1W_2$, where $W_1\in\bR^{d_1\times k}$ and $W_2\in\bR^{k\times d_2}$ for some search rank $k\geq r$. \ref{BM} with $\ell_1$-loss adapted to matrix sensing is thus formulated as:
	\begin{align}
		&\!\!\!\min_{W_1\in\bR^{d_1\times k}, W_2\in\bR^{k\times d_2}} \! f^{\mathrm{s}}_{\ell_1}(\fW) \!:=\! \frac{1}{m}\sum_{i=1}^m\left|y_i\!-\!\langle A_i, W_1W_2\rangle\right|,\ \text{where}\ A_i[{\alpha, \beta}]\stackrel{iid}{\sim} \cN(0,1),\ \forall i,{\alpha, \beta}.\tag{\texttt{MS-asym}}\label{BM-sensing-asym}
	\end{align}
 
	In \textit{symmetric} matrix sensing, the true solution $\trueX$ is additionally assumed to be positive semidefinite with dimension $d\times d$. In this setting, the true solution can be modeled as $\trueX = \trueW{\trueW}^\top$ with $\trueW\in \bR^{d\times r}$ and~\ref{BM} can be reformulated as:
 % \vspace{-0.5cm}
	\begin{align}
		&\min_{W\in\bR^{d\times k}} f^{\mathrm{s}}_{\ell_1}(W) \! :=\! \frac{1}{m}\sum_{i=1}^m\left|y_i-\left\langle A_i, WW^\top\right\rangle\right|, \ \text{where}\ A_i[{\alpha, \beta}]\stackrel{iid}{\sim} \cN(0,1),\forall i,{\alpha, \beta}.\tag{\texttt{MS-sym}}\label{BM-sensing-sym}
	\end{align}
% \vspace{-1.3cm}
 
	Another important subclass of low-rank matrix recovery is matrix completion, where the linear operator $\cA(\cdot)$ is assumed to be element-wise projection.
	\begin{assumption}[Matrix completion model]\label{asp_completion}
		Each index pair $({\alpha, \beta})$ belongs to a measurement set ${\Omega}$ with a sampling probability $0\leq s\leq 1$, i.e., $\bP(({\alpha,\beta})\in{\Omega}) = s$ for every $1\leq {\alpha}\leq d_1$ and $1\leq {\beta}\leq d_2$. Moreover, for every $({\alpha',\beta'})\in {\Omega}$, there exists $1\leq i\leq |{\Omega}|$ such that the measurement matrix $A_i$ is defined as $A_i[{\alpha,\beta}] = 1$ if $({\alpha,\beta}) = ({\alpha',\beta'})$ and $A_i[{\alpha,\beta}] = 0$ otherwise.
	\end{assumption}
 {Recall that $y_i$ denotes the $i$-th measurement.}
	Based on the above model, we define $Y\in\mathbb{R}^{d_1\times d_2}$ as $Y[{\alpha,\beta}] = y_i$ if $A_i[{\alpha,\beta}] = 1$ for some $1\leq i\leq m$ and $({\alpha,\beta})\in {\Omega}$, and $Y[{\alpha,\beta}] = 0$ otherwise. Similarly, we define $E\in\mathbb{R}^{d_1\times d_2}$ as $E[{\alpha,\beta}] = \epsilon_i$ if $A_i[{\alpha,\beta}] = 1$ for some $1\leq i\leq m$ and $({\alpha,\beta})\in {\Omega}$, and $E[{\alpha,\beta}] = 0$ otherwise.
	If the measurements follow matrix completion model, \ref{BM} can be written as:
 % \vspace{-0.5cm}
	\begin{align}
		&\min_{W_1\in\bR^{d_1\times k}, W_2\in\bR^{k\times d_2}} f^{\mathrm{c}}_{\ell_1}(\fW):= \frac{1}{m}\!\!\sum_{({\alpha,\beta})\in{\Omega}}\left|Y[{\alpha,\beta}]-(W_1W_2)[{\alpha,\beta}]\right|.\tag{\texttt{MC-asym}}\label{BM-completion-asym}
	\end{align}
 % \vspace{-1.3cm}
 
	In symmetric matrix completion, the true solution $\trueX$ is additionally assumed to be positive semidefinte with dimension $d\times d$. Under this assumption, \ref{BM} can be written as:
 % \vspace{-0.5cm}
	\begin{align}\tag{\texttt{MC-sym}}\label{BM-completion-sym}
		\min_{ W\in \bR^{d\times k}} \! f^{\mathrm{c}}_{\ell_1}\!(W) := \frac{1}{m}\!\!\sum_{({\alpha,\beta})\in{\Omega}}\!\left|Y[{\alpha,\beta}]\!-\!\left(WW^\top\right)\![{\alpha,\beta}]\right|\!.
	\end{align}
 % \vspace{-1.3cm}
 
	Finally, we present our noise model.
	\begin{assumption}[Noise Model]
		\label{assumption::general-noise}
		Each measurement is independently corrupted with noise with a corruption probability $0\leq p\leq 1$. Let the set of noisy measurements be denoted as $\cS$. For each entry $i\in \cS$, the value of $\err_i$ is drawn from a distribution ${P_o}$. Moreover, a random variable $\zeta$ under the distribution ${P_o}$ satisfies $\bP(|\zeta|\geq t_0)\geq p_0$ for some constants $0< t_0, p_0\leq 1$.
	\end{assumption}
	Our considered noise model is generic and includes all the ``typical'' noise models, including Gaussian and outlier noise models. Intuitively, it only requires a nonzero mass at a nonzero value. For example, suppose that for every $i\in \cS$, $\epsilon_i\sim \cN(0,\sigma^2)$. A basic anti-concentration inequality implies that $\bP(|\epsilon_i|\geq t_0)=1-\bP(|\err_i|/\sigma<t_0/\sigma)\geq 1-2t_0/\sigma$, which satisfies  \Cref{assumption::general-noise} with $(t_0,p_0) = \left((4\sigma)^{-1}, 1/2\right)$. Similarly, a randomized sparse outlier noise is expected to satisfy  \Cref{assumption::general-noise}.
	
	\subsection{Local Landscape of Matrix Sensing}

 \paragraph{Symmetric matrix sensing.} Our next two theorems establish the necessary and sufficient conditions for determining the local optimality, {or lack thereof,} of true solutions in symmetric matrix sensing.

 \begin{theorem}[Sub-optimality of true solutions for symmetric matrix sensing]\label{thm_no_benign_sym}
		Consider~\ref{BM-sensing-sym} with measurement matrices satisfying  \Cref{asp_sensing}. Suppose that the noise satisfies \Cref{assumption::general-noise}. Suppose that $k> r$ and $m\lesssim {pp_0d(k-r)}$. With probability at least $1-\exp(-\Omega(d(k-r)))-\exp(-\Omega(pp_0m))$, none of the true solutions in $\cW$ are local optima of $f^{\mathrm{s}}_{\ell_1}(W)$. {More precisely, for every $\trueW\in\cW$, there exists a descent direction $\Delta W$ satisfying $\norm{\Delta W}_F=1$ such that, for every $0<\gamma^2\lesssim \frac{t_0}{\sqrt{d}}$, we have
			\begin{equation}\label{eq_ms_unidentifiability}
				 f^{\mathrm{s}}_{\ell_1}(W^\star{+\gamma\Delta W})- f^{\mathrm{s}}_{\ell_1}(W^\star)\lesssim -\gamma^2.
			\end{equation}}
	\end{theorem}
	{According to \Cref{thm_no_benign_asym}, none of the true solutions emerge as local optima of $f^{\mathrm{s}}_{\ell_1}$, unless the number of measurements scale with the search rank $k$.} In particular, suppose that $2r\leq k$ and the parameters of the noise model are fixed. Then, \Cref{thm_no_benign_sym} shows that, with high probability, none of the true solutions are local optima, so long as $m\lesssim dk$. 
	
	\begin{theorem}[Optimality of true solutions for symmetric matrix sensing]\label{thm_critical_sensing_sym}
		Consider~\ref{BM-sensing-sym} with measurement matrices satisfying  \Cref{asp_sensing}. Suppose that the noise satisfies  \Cref{assumption::general-noise} with a corruption probability $0\leq p< 1/2$. Moreover, suppose that $k\geq r$. The following statements hold:
		\begin{itemize}
			\item Suppose that $m\gtrsim \frac{dr}{(1-2p)^4}$. With probability at least $1-\exp(-\Omega(dr))$, {for every $W^\star\in \cW$ and} $\gamma\geq 0$, we have
			{\begin{equation*}
			    \begin{aligned}
			        &\min_{\norm{\Delta W}_F=1}\left\{ f^{\mathrm{s}}_{\ell_1}(W^\star+\gamma\Delta W)- f^{\mathrm{s}}_{\ell_1}(W^\star)\right\}
                    \gtrsim  -\left(\sqrt{\frac{2}{\pi}}+\sqrt{\frac{dk}{m}}\right)\gamma^2.
			    \end{aligned}
			\end{equation*}}
			\item Suppose that $m\gtrsim \frac{dk}{(1-2p)^4}$. With probability at least $1-\exp(-\Omega(dk))$, all true solutions coincide with the global optima. More precisely, we have
			\begin{equation*}
			    \trueW\in\cW\quad \iff \quad \trueW\in\argmin_{W}f^{\mathrm{s}}_{\ell_1}(W).
			\end{equation*}
        \end{itemize}
	\end{theorem}
 {We note that the second statement of the above theorem has been previously proven by~\cite{ding2021rank}. 
 However, we include it here for completeness.
 
 Combining \Cref{thm_no_benign_sym} with the first statement of \Cref{thm_critical_sensing_sym} leads to the following key corollary.
 \begin{corollary}
 \label{cor::sym}
     Consider~\ref{BM-sensing-sym} with measurement matrices satisfying  \Cref{asp_sensing}. Suppose that the noise satisfies  \Cref{assumption::general-noise} with $p_0 = \Theta(1)$ and $p = \Theta(1)<1/2$. Additionally, suppose that $k\gtrsim r$ and $dr\lesssim m\lesssim dk$. Then, with probability at least $1-\exp(-\Omega(dk))-\exp(-\Omega(m))$, all true solutions are critical and strict saddle points.
 \end{corollary}}
 {The above results} show that the landscape of symmetric matrix sensing with $\ell_1$-loss and outlier noise (with $p<1/2$) undergoes two critical transitions: When $m\lesssim dr$, exact recovery is impossible due to the existing information theoretical barriers~\cite{candes2011tight}. When $dr\lesssim m\lesssim dk$, all true solutions emerge as {strict saddle points}. We note that this is precisely the regime where the sub-gradient method empirically converges to one of the true solutions~\cite{ma2023global}. 
 As soon as $m\gtrsim dk$, all true solutions coincide with the global optima. 

\paragraph{Asymmetric matrix sensing.}
 Next, we consider the asymmetric matrix sensing with $\ell_1$-loss, formulated as~\ref{BM-sensing-asym}. A key distinction in the asymmetric setting lies in the existence of \textit{rank-imbalanced} solutions, i.e., true solutions $(\trueW_1, \trueW_2)\in \cW$ with $2r< \rank(\trueW_1)+\rank(\trueW_2)\leq 2k$. 
 To illustrate this, consider the SVD of the true solution $\trueX = U^\star\Sigma^\star {V^{\star\top}}$ and suppose $k>r$. Indeed, the point 
$\fW^\star=(\trueW_1,\trueW_2)$ with $\trueW_1 = \begin{bmatrix}
     U^\star \Sigma^\star & 0_{d_1\times (k-r)} 
 \end{bmatrix}
 , \trueW_2 = \begin{bmatrix}
     {V^\star} & Z
 \end{bmatrix}^{\top}$ is a true solution 
 for any arbitrary $Z\in \bR^{d_2\times (k-r)}$. Moreover, $\rank(\trueW_1)+\rank(\trueW_2)=r+k>2r$, provided that $Z$ is chosen such that $\trueW_2$ is full row-rank. 
 In contrast, any true solution $\trueW$ for the symmetric matrix recovery inevitably must satisfy $\rank(\trueW)=r$, and hence, is rank-balanced. 
 As we will demonstrate later, the existence of rank-imbalanced solutions drastically changes the local landscape of asymmetric matrix recovery.

To formalize our ideas, we define
{\begin{equation*}
    \begin{aligned}
        &\cW(s) = \left\{(W_1,W_2)\in \cW: \rank(W_1)+\rank(W_2)\leq s\right\}.
    \end{aligned}
\end{equation*}}
To analyze the landscape of the asymmetric matrix sensing, we divide the set of true solutions into two subsets: $\cW = \cW(2k/3)\cup\left( \cW\backslash\cW(2k/3)\right)$. We call the set $\cW\backslash\cW(2k/3)$ the set of \textit{rank-imbalanced solutions}. To see this, suppose that $k>3r$. Then, for any $\fW^\star = (\trueW_1,\trueW_2)\in \cW\backslash\cW(2k/3)$, we have $\rank(\trueW_1)+\rank(\trueW_2)>2r$. In other words, for any true solution in $\cW\backslash\cW(2k/3)$, at least one of the factors has a rank that is higher than the true rank.  

	\begin{theorem}[Sub-optimality of true solutions for asymmetric matrix sensing]\label{thm_no_benign_asym}
		Consider~\ref{BM-sensing-asym} with measurement matrices satisfying  \Cref{asp_sensing}. Suppose that the noise satisfies \Cref{assumption::general-noise} and $k\geq r$. With probability at least $1-\exp(-\Omega(p_0pm))$, the following statements hold:
		\begin{itemize}
            \item Suppose that $m\lesssim \max\{d_1,d_2\}k$. Then, none of the true solutions in $\cW(2k/3)$ are local optima. {More precisely, for every $\fW^\star\in\cW(2k/3)$, there exists a descent direction $\Delta \fW$ satisfying $\norm{\Delta \fW}_F=1$ such that, for every $0<\gamma\lesssim \sqrt{\frac{mpp_0}{\max\{d_1,d_2\}k}}t_0$, we have
            \begin{equation*}
                \begin{aligned}
                    & f^{\mathrm{s}}_{\ell_1}(\fW^*{+\gamma\Delta \fW})- f^{\mathrm{s}}_{\ell_1}(\fW^\star)\lesssim -\sqrt{\frac{\min\{d_1,d_2\}pp_0}{m}}\gamma^2.
                \end{aligned}
            \end{equation*}}
            \item Suppose that $m\lesssim \max\{d_1,d_2\}k$ and there exist constants $0<c_1\leq c_2<1$ such that $c_1\leq \frac{\min\{d_1,d_2\}}{\max\{d_1,d_2\}}\leq c_2$. Then, none of the true solutions in $\cW\backslash\cW(2k/3)$ are critical points. More precisely, {for every $\fW^\star\in\cW\backslash\cW(2k/3)$, there exists a descent direction $\Delta \fW$ satisfying $\norm{\Delta \fW}_F=1$ such that, for every $0\leq \gamma\lesssim \sqrt{\frac{mpp_0}{\max\{d_1,d_2\}k}}\cdot\frac{t_0}{\max\left\{\sigma_{\min,+}(W_1^\star), \sigma_{\min,+}(W_2^\star)\right\}}$, we have
            \begin{equation*}
                \begin{aligned}
                    f^{\mathrm{s}}_{\ell_1}(\fW^\star+{\gamma}\Delta \fW) - f^{\mathrm{s}}_{\ell_1}(\fW)
    \lesssim -\left(\min\left\{\sigma_{\min,+}(W_1^\star), \sigma_{\min,+}(W_2^\star)\right\}\sqrt{\frac{\min\{d_1,d_2\}kpp_0}{m}} \right)\cdot\gamma.
                \end{aligned}
            \end{equation*}}
   \end{itemize}
   \end{theorem}
{\Cref{thm_no_benign_asym} demonstrates that, similar to the symmetric case, a necessary condition for the true solutions to emerge as local optima in asymmetric matrix sensing is that the number of measurements must scale with the search rank $k$.} Moreover, \Cref{thm_no_benign_asym} provides another negative result: no matter how small the corruption probability $p>0$ is, none of the rank-imbalanced solutions are critical points of~\ref{BM}, provided that $m\lesssim \max\{d_1,d_2\}k$ and $c_1\leq \frac{\min\{d_1,d_2\}}{\max\{d_1,d_2\}}\leq c_2$. As a result, the rank-imbalanced solutions remain beyond the reach of any first-order optimization method designed to converge to a critical point. Finally, we note that the requirement $c_1\leq \frac{\min\{d_1,d_2\}}{\max\{d_1,d_2\}}\leq c_2$ may be a result of our current proof technique and could potentially be relaxed with a more refined analysis. 

The second statement of \Cref{thm_no_benign_asym} is in contrast to the landscape of symmetric matrix sensing, where it is shown that all true solutions become critical if the sample size exceeds $dr$. Nonetheless, our next theorem shows that, unlike the rank-imbalanced solutions, the \textit{rank-balanced} solutions do in fact emerge as critical points, provided that $m\gtrsim \max\{d_1,d_2\}r$.

\begin{theorem}[Optimality of true solutions for asymmetric matrix sensing]\label{thm_benign_l1}
		Consider~\ref{BM-sensing-asym} with measurement matrices satisfying  \Cref{asp_sensing}. Suppose that the noise satisfies  \Cref{assumption::general-noise} with a corruption probability $0\leq p< 1/2$. Moreover, suppose that $k\geq r$. The following statements hold:
		\begin{itemize}
			\item Suppose that $m\gtrsim \frac{\max\{d_1,d_2\}r}{(1-2p)^4}$. With probability at least $1-\exp(-\Omega(\max\{d_1,d_2\}r))$, {for every $\fW^\star\in \cW(2r)$ }and $\gamma\geq 0$, we have
			{\begin{equation*}
			    \begin{aligned}
			        &\min_{\norm{\Delta \fW}_F=1}\left\{ f^{\mathrm{s}}_{\ell_1}(\fW^\star+\gamma\Delta\fW)- f^{\mathrm{s}}_{\ell_1}(\fW^\star)\right\}\gtrsim  -\left(\sqrt{\frac{2}{\pi}}+\sqrt{\frac{\max\{d_1,d_2\}k}{m}}\right)\gamma^2.
			    \end{aligned}
			\end{equation*}}
			\item Suppose that $m\gtrsim \frac{\max\{d_1,d_2\}k}{(1-2p)^4}$. With probability at least $1-\exp(-\Omega(\max\{d_1,d_2\}k))$, all true solutions coincide with the global optima. More precisely, we have
			\begin{equation*}
			    \fW^\star\in\cW\quad \iff \quad \fW^\star\in\argmin_{\fW}f^{\mathrm{s}}_{\ell_1}(\fW).
			\end{equation*}
		\end{itemize}
	\end{theorem} 
{
 Combining the above two theorems leads to the following key corollary.
 \begin{corollary}
  \label{cor::asym}
    Consider~\ref{BM-sensing-asym} with measurement matrices satisfying \Cref{asp_sensing}. Suppose that the noise satisfies  \Cref{assumption::general-noise} with $p_0=\Theta(1)$ and $p=\Theta(1)< 1/2$.  Additionally, suppose that $k\gtrsim r$ and ${\max\{d_1,d_2\}r}\lesssim m\lesssim \max\{d_1,d_2\}k$. Then, with probability at least $1-\exp(-\Omega(m))-\exp(-\Omega(\max\{d_1,d_2\}r))$, all true solutions in $\cW(2r)$ are critical and strict saddle points.
 \end{corollary}}
 {This corollary reveals that, unlike symmetric matrix sensing, only rank-balanced solutions emerge as strict saddle points of~\ref{BM-sensing-asym} when ${\max\{d_1,d_2\}r}\lesssim m\lesssim \max\{d_1,d_2\}k$.}
 
 \subsection{Local Landscape of Matrix Completion}
	
 \paragraph{Symmetric matrix completion.} Next, we study the local landscape of symmetric matrix completion around its true solutions. 

 \begin{theorem}[Sub-optimality of true solutions for symmetric matrix completion]\label{thm_completion_sym}
		Consider~\ref{BM-completion-sym} with measurement matrices satisfying  \Cref{asp_completion}. Suppose that $k>r$ and there exists $({\alpha, \alpha})\in{\Omega}$ such that $E[{\alpha, \alpha}]\geq t_0$. Then, with probability at least $1-\exp(-\Omega(sd_1d_2))$, none of the true solutions in $\cW$ are local optima of $f^{\mathrm{c}}_{\ell_1}(W)$. More precisely, {for every $W^\star\in\cW$, there exists a descent direction $\Delta W$ satisfying $\norm{\Delta W}_F=1$ such that, for every $0<\gamma\lesssim \sqrt{t_0}$, we have
			\begin{equation*}
			     f^{\mathrm{c}}_{\ell_1}(W{^*+\gamma\Delta W})- f^{\mathrm{c}}_{\ell_1}(W^\star)\lesssim -\frac{\gamma^2}{sd^2}.
			\end{equation*}}
	\end{theorem}
	
	Unlike our results for the matrix sensing problem, the necessary condition for the local optimality of the true solutions for symmetric matrix completion does not rely on the randomness of the noise: none of the true solutions are local or global so long as $k>r$ {(e.g., $k=r+1$)} and at least one of the diagonal entries of $\trueX$ is observed with a positive noise. The latter condition is indeed very mild and holds with a high probability, provided that the noise takes a positive value with a nonzero probability. For instance, suppose that $\bP(\err_i\geq t_0)\geq p_+>0$. Then, the probability of $E[{\alpha, \alpha}]\geq t_0$ for some $({\alpha, \alpha})\in{\Omega}$ is at least $1-(1-spp_+)^d$. 
 Moreover, this result holds even in a near-ideal scenario, where the sampling probability $s$ is equal to 1 (i.e., the entire matrix $\trueX$ is observed). 

\paragraph{Asymmetric matrix completion.} Our next theorem extends the above result to the asymmetric matrix completion under additional conditions on $\trueX$. To achieve this goal, we first introduce the notion of coherence for a matrix.
\begin{definition}
    Consider a rank-$r$ matrix $M\in \bR^{d_1\times d_2}$ and its SVD $M = U\Sigma {V}^\top$, where $U\in \bR^{d_1\times r}$ and $V\in \bR^{d_2\times r}$ are the left and right singular vectors corresponding to its nonzero singular values. The matrix $M$ is called \textbf{coherent} if 
 \begin{equation*}
     \max\left\{\max_{1\leq i\leq r}\norm{U e_i}_\infty, \max_{1\leq i\leq r}\norm{V e_i}_\infty\right\} = 1.
 \end{equation*}
\end{definition}
 Based on the above definition, a matrix is coherent if at least one of its left or right singular vectors is aligned with a standard unit vector. {It is easy to see that if the ground truth $\trueX$ is coherent, the sampling probability that is necessary for the recovery of $\trueX$ can grow to a value of 1.}\footnote{To illustrate this point, let us examine a scenario where both the left and right singular vectors of $\trueX$ align perfectly with the standard basis. As a result, $\trueX$ simplifies to a diagonal matrix. In such settings, achieving an exact recovery necessitates the observation of nearly every element within $\trueX$.} Our next theorem shows that, in the presence of noise, even a sampling probability of $s=1$ is insufficient to ensure that a coherent ground truth emerges as a critical point.
	
	\begin{theorem}[Sub-optimality of true solutions for asymmetric matrix completion]\label{thm_completion_asym}
		Consider~\ref{BM-completion-asym} with measurement matrices satisfying  \Cref{asp_completion}. Suppose that the noise satisfies \Cref{assumption::general-noise}. With probability at least $ 1-\exp(-\Omega(spp_0\max\{d_1,d_2\}(k-r)))$, the following statements hold:
		\begin{itemize}
			\item Suppose that $k> r$. {There exists a non-empty subset of the true solutions $\bar{\cW}\subseteq \cW$ that are non-critical for $f^{\mathrm{c}}_{\ell_1}(\fW)$}.  {More precisely, for every $\fW^\star\in \bar{\cW}$, there exists a descent direction $\Delta \fW$ satisfying $\norm{\Delta \fW}_F=1$ such that, for every $0<\gamma\lesssim t_0$, we have
			\begin{equation*}
			    \begin{aligned}
			        & f^{\mathrm{c}}_{\ell_1}(\fW{^*+\gamma\Delta \fW})- f^{\mathrm{c}}_{\ell_1}(\fW^\star)\lesssim  -\sqrt{\frac{k-r}{d_2}}\cdot\sqrt{\frac{pp_0}{sd_1d_2}}\cdot\gamma.
			    \end{aligned}    
			\end{equation*}}
   \item Suppose that $k>2r$. None of the true solutions in $\cW(k-1)$ are local optima of $f^{\mathrm{c}}_{\ell_1}(\fW)$. More precisely, {for every $\fW^\star\in\cW(k-1)$, there exists a descent direction $\Delta \fW$ satisfying $\norm{\Delta \fW}_F=1$ such that, for every $0<\gamma\lesssim \sqrt{t_0}$, we have
   \begin{equation*}
       f^{\mathrm{c}}_{\ell_1}(\fW{^*+\gamma\Delta\fW})- f^{\mathrm{c}}_{\ell_1}(\fW^\star)\lesssim -\frac{\gamma^2}{sd_1d_2}.
   \end{equation*}}
   \item Suppose that $k\geq r$ and $\trueX$ is coherent. None of the true solutions with a bounded norm are critical points of $f^{\mathrm{c}}_{\ell_1}(\fW)$. More precisely, for any radius $\Gamma<\infty$, let $\cB_{\mathrm{op}}(\Gamma) = \{(W_1,W_2): \max\{\norm{W_1},\norm{W_2}\}\leq \Gamma\}$. {Then, for every $\fW^\star\in\cW\cap \cB_{\mathrm{op}}(\Gamma)$, there exists a descent direction $\Delta \fW$ satisfying $\norm{\Delta \fW}_F=1$ such that, for every $0<\gamma\lesssim t_0$, we have
   \begin{equation*}
       \begin{aligned}
           & f^{\mathrm{c}}_{\ell_1}(\fW{^*+\gamma\Delta\fW})- f^{\mathrm{c}}_{\ell_1}(\fW^\star)\lesssim  -\sqrt{\frac{pp_0}{sd_1d_2^2}}\cdot\frac{\sigma_r(\trueX)\gamma}{\Gamma}.
       \end{aligned}
   \end{equation*}}
		\end{itemize}
	\end{theorem}
 \Cref{thm_completion_asym} shows that, similar to the asymmetric matrix sensing, non-critical true solutions are ubiquitous in the asymmetric matrix completion. In particular, the first statement showcases that, even if the search rank is only slightly over-parameterized (e.g., $k=r+1$), there exists at least one true solution that is non-critical. This result holds with a high probability, even with the sampling probability of $s=1$ provided that $pp_0>0$. The second statement, on the other hand, shows that when $k>2r$, none of the true solutions with $\rank(\trueW_1)+\rank(\trueW_2)<k$---including the rank-balanced solutions that satisfy $\rank(\trueW_1) = \rank(\trueW_1) = r$---are local optima of the loss function. We note that this result does not refute the possibility of existing highly rank-imbalanced solutions---in particular, those with $\rank(\trueW_1)+\rank(\trueW_2)\geq k$---as local optima. Although our result does not refute this possibility, we conjecture that highly rank-imbalanced solutions are unlikely to emerge as local optima. We leave a rigorous proof of this conjecture for future work. Finally, the last statement shows that the asymmetric matrix completion with coherent ground truth has the most undesirable landscape; even in the \textit{exactly-parameterized regime} with $k=r$ and with sampling probability of $s=1$, none of the true solutions are critical points, so long as $pp_0>0$. This complements the existing results for the asymmetric matrix completion, which show that the ground truth can be recovered exactly, provided that it is \textit{incoherent}, i.e., it does not have a significant alignment with the standard basis vectors~\cite{candes2011robust, chandrasekaran2011rank, chen2015incoherence}.

	The rest of the paper is organized as follows. 
 in Section~\ref{sec_prelim}, we introduce and elaborate on a collection of tools from variational analysis and random processes that will serve as foundational elements in our subsequent analysis. 
 In Section~\ref{sec_sym}, we prove our results for the symmetric matrix sensing and symmetric matrix completion (Theorems~\ref{thm_no_benign_sym} and~\ref{thm_completion_sym}). In Sections~\ref{seC_Asym}, we extend our analysis to the asymmetric settings (Theorems~\ref{thm_no_benign_asym} and~\ref{thm_completion_asym}). In Section~\ref{sec_lb}, we present the proofs of our provided lower bounds (Theorems~\ref{thm_critical_sensing_sym} and~\ref{thm_benign_l1}).  {In \Cref{sec::strict-saddle}, we provide the proofs of the strict saddle property of the true solutions (\Cref{cor::sym,cor::asym}). Finally, in~\Cref{sec::discussion}, we briefly discuss the algorithmic implications of our results and some future directions.}

\section{Preliminaries}\label{sec_prelim}
In Subsection~\ref{subsec_var_analysis}, we review basic definitions and results from variational analysis that will be instrumental in developing our results. Then, in Subsections~\ref{subsec_orlicz} to~\ref{subsec_useful}, we lay the groundwork for certain 
concentration results for different classes of random variables and random processes that will be used throughout our proofs.

\subsection{Basic Properties from Variational Analysis}\label{subsec_var_analysis}
	For a function $f: \bR^{n}\to \bR$, a point $\bar x$ is called a \textit{global optimal} if it corresponds to its global minimizer. Moreover, a point $\bar x$ is called a \textit{local optimal} if it corresponds to the minimum of $f(x)$ within an open ball centered at $\bar x$. The directional derivative of $f$ at point $x$ in the feasible direction $d$ is defined as 
	\begin{equation}\label{eq_dir_der}
	    f'(x,d) = \lim_{t\to 0^+}\frac{f(x+td)-f(x)}{t},\nonumber
	\end{equation}
	provided that the limit exists. If $f'(x,d)<0$, then $d$ is called a \textit{descent direction}. For a locally Lipschitz function $f$, the {\it Clarke generalized directional derivative} at the point $x$ in the feasible direction $d$ is defined as
	\begin{equation*}
		f^\circ(x,{d}) := \underset{\begin{subarray}{c}
				y\rightarrow x\\
				t\to 0^+
		\end{subarray}}{\lim\sup}\frac{f(y+t{d})-f(y)}{t}
	\end{equation*}
	provided that the limit exists. 
 It is a well-known fact that $f'(x,d)\leq f^\circ(x,d)$ for every direction $d$ if both $f'(x,d)$ and $f^\circ(x,d)$ exist~\cite[Proposition 1.4]{clarke1975generalized}. A function $f$ is called \textit{subdifferentially regular} if $f'(x,d)= f^\circ(x,d)$ for every point $x$ and direction $d$~\cite[Definition 2.3.4]{clarke1990optimization}.
	The {\it Clarke subdifferential} of $f$ at ${x}$ is defined as the following set (see~\cite[Definition 1.1 and Proposition 1.4]{clarke1975generalized}):
	\begin{equation}\label{partial}
		\partial f({x}) := \{{\psi} | f^\circ(x,{d})\geq\langle{\psi}, {d}\rangle, \forall{d}\in\mathbb{R}^{n}\}.
	\end{equation} 
	A point $\bar{x}$ is called \textit{critical} if $0\in\partial f(\bar{x})$, or equivalently, $f^\circ(\bar x,{d})\geq 0$ for every feasible direction $d$.
	The following properties of the critical points are adapted from~\cite{li2020understanding} and will be used in our subsequent arguments.
	\begin{lemma}\label{lem_critical}
		For a real-valued locally Lipschitz function $f: \bR^n\to \bR$, we have:
		\begin{enumerate}
			\item Every local minimum of $f$ is critical.
			\item Given a point $x$, suppose that $f'(x,d)\geq 0$ for every feasible $d\in \bR^n$.
			Then, $x$ is a critical point of $f$.
			\item Suppose that $f = \max_{1\leq i\leq m} g_i$ for some smooth functions $g_1,\dots,g_m$. {Then, $f'(x,d)$ exists for every point $x$ and direction $d$.} Moreover,  given some point $x$, suppose that $f'(x,d)< 0$ for some direction $d\in \bR^n$. Then, $x$ is not a critical point of $f$.
		\end{enumerate}
	\end{lemma}
	\proof{Proof}
		The first property is a direct consequence of~\cite[Proposition 2.3.2]{clarke1990optimization}. The second property follows from the basic inequality $f'(x,d)\leq f^\circ(x,d)$~\cite[Proposition 1.4]{clarke1975generalized} and the definition of the critical points. {Finally, we prove the third property. Note that the existence of $f'(x,d)$ for every point $x$ and direction $d$ can be easily verified via the definition of $f$ and directional derivative. }Moreover, due to~\cite[Example 7.28]{rockafellar2009variational}, $f$ is differentiably regular, and hence, $f^\circ(x,d) = f'(x,d)$ for every direction $d$. This implies that $x$ is not critical due to the existence of a direction $d$ for which $f^\circ(x,d) = f'(x,d)<0$. \hspace*{\fill}$\Box$\par
	\endproof

Utilizing the above properties, the following crucial lemma provides a sufficient condition for the non-criticality of a point $\bar \fW$ for the $\ell_1$-loss $f_{\ell_1}(\fW)$.
	\begin{lemma}\label{lem_critical_mr}
		{The directional derivative $f'_{\ell_1}(\bar\fW, \Delta \fW)$ exists for every point $\fW$ and direction $\Delta\fW$.} Moreover, given a point $\bar\fW$, suppose that there exists a direction $\Delta\fW$ such that $f'_{\ell_1}(\bar\fW, \Delta \fW)<0$. Then, $\bar\fW$ is not a critical point of $f_{\ell_1}(\fW)$.
	\end{lemma}
	\proof{Proof}
		We first define $\cM$ as the class of binary functions $\sigma: \{1,\dots,m\}\to \{-1,+1\}$. 
		It is easy to see that $f_{\ell_1}(\fW) = \max_{\sigma\in\cM} g_{\sigma}(\fW)$, where $g_\sigma(\fW) = \sum_{i\in m}\sigma(i)\left(y_i-\inner{A_i}{W_1W_2}\right)$. Note that $g_\sigma(\fW)$ is smooth for any choice of $\sigma\in \cM$. Therefore, the third property of \Cref{lem_critical} can be invoked to complete the proof. \hspace*{\fill}$\Box$\par
	\endproof
  {
  Finally, following \cite{davis2022proximal}, we call a critical point $\bar{x}$ of a non-differentiable function $f$ \textit{strict saddle} if there exists a direction $v$ such that $f$ decreases quadratically along $v$. This is formally defined as follows.
 
 \begin{definition}[Strict saddle point]\label{def_ssp}
     A critical point $\bar x$ of the function $f$ is called \textit{strict saddle} if there exist $v\in \bR^d, c>0, \bar\gamma>0$ such that, for every $0\leq \gamma\leq \bar\gamma$, we have
     \begin{equation*}
         f(\bar x+\gamma v)-f(\bar x)\leq -c\gamma^2.
     \end{equation*}
 \end{definition}}

	\subsection{Orlicz Random Variables and Processes}\label{subsec_orlicz}
	First, we define the notion of Orlicz norm (see \cite[Section 5.6]{wainwright2019high} and \cite[Section 11.1]{ledoux1991probability} for more details).
	\begin{definition}[Orlicz function and Orlicz norm]
		A function $\psi:\bR_+\to \bR_+$ is called an Orlicz function if $\psi$ is convex, increasing, and satisfies
		\begin{equation*}
			\psi(0)=0, \quad \psi(x)\to \infty \text{ as } x\to \infty. 
		\end{equation*}
		For a given Orlicz function $\psi$, the Orlicz norm of a random variable $X$ is defined as
		\begin{equation*}
			\norm{X}_{\psi}:=\inf\{t>0:\bE[\psi(|X|/t)]\leq 1\}.
		\end{equation*}
	\end{definition}
We will use Orlicz functions and norms to study the concentration of both sub-Gaussian and sub-exponential random variables, which are defined as follows.
	\begin{definition}[Sub-Gaussian and sub-exponential random variables]
		\label{lem::concentration-orlicz-rv}
		Define the Orlicz function $\psi_p(x)=\exp\{x^p\}-1$.
		A random variable $X$ is called {{sub-Gaussian}} if $\norm{X}_{\psi_2}<\infty$. Moreover, a random variable $X$ is called  {{sub-exponential}} if $\norm{X}_{\psi_1}<\infty$. Accordingly, the quantities $\norm{X}_{\psi_2}$ and $\norm{X}_{\psi_1}$ refer to the sub-Gaussian norm and sub-exponential norm of $X$, respectively.
	\end{definition}
	Lemmas~\ref{lem_phi}-\ref{lem_mult} establish basic properties of sub-Gaussian and sub-exponential random variables.
 \begin{sloppypar}
	\begin{lemma}[Tail bounds for sub-Gaussian and sub-exponential random variables \protect{\cite[Section 5.6]{wainwright2019high}}]\label{lem_phi}
		For a random variable $X$ with finite Orlicz norm $\norm{X}_{\psi_p}< \infty$ where $\psi_p(x)=\exp\{x^p\}-1$ and $p\geq 1$, we have 
		\begin{equation*}
			\bP\left(|X|\geq \delta\right)\leq 2\exp\left(-\frac{\delta^p}{\norm{X}_{\psi_p}^p}\right).
		\end{equation*}
	\end{lemma}
 \end{sloppypar}
	\begin{lemma}[Sum of independent random variables  \protect{\cite[Theorems 2.6.2 and 2.8.1]{vershynin2018high}}]
		\label{lem::concentration-sum-independent}
		Let $X_1, \cdots, X_n$ be zero-mean independent random variables.
		\begin{itemize}
			\item If each $X_i$ is sub-Gaussian with sub-Gaussian norm $\norm{X_i}_{\psi_2}$, we have 
			\begin{equation*}
				\bP\left(\left|\frac{1}{n}\sum_{i=1}^{n}X_i\right|\geq\delta\right)\leq 2\exp\left(-\frac{n^2\delta^2}{\sum_{i=1}^{n}\norm{X_i}_{\psi_2}^2}\right).
			\end{equation*}
			\item If each $X_i$ is sub-exponential with sub-exponential norm $\norm{X_i}_{\psi_1}$, we have
			\begin{equation*}
				\begin{aligned}
				    &\bP\left(\left|\frac{1}{n}\sum_{i=1}^{n}X_i\right|\geq\delta\right)\leq 2\exp\!\left(-n\min\left\{\frac{\delta^2}{\max_{i}\norm{X_i}_{\psi_1}^2},\frac{\delta}{\max_{i}\norm{X_i}_{\psi_1}}\right\}\right).
				\end{aligned}
			\end{equation*} 
		\end{itemize}
	\end{lemma}
	\begin{lemma}[Lemma 2.6.8 in \cite{vershynin2018high}]
		\label{lem::centering}
		Suppose that $X$ is sub-Gaussian. We have $\norm{X-\bE[X]}_{\psi_2}\leq C\norm{X}_{\psi_2}$ for some constant $C>0$.
	\end{lemma}
	\begin{lemma}[Lemma 2.7.7 in \cite{vershynin2018high}]\label{lem_mult}
		Let $X, Y$ be sub-Gaussian random variables. Then $XY$ is sub-exponential with its sub-exponential norm satisfying $\norm{XY}_{\psi_1}\leq \norm{X}_{\psi_2}\norm{Y}_{\psi_2}$.
	\end{lemma}
	Next, we introduce the notion of Orlicz process.
	\begin{definition}[Orlicz process, Definition 5.35 in \cite{wainwright2019high}]
		A set of zero-mean random variables $\{X_t, t \in \cT\}$ (also known as a stochastic process) is an Orlicz process with Orlicz norm $\psi_p$ (also known as $\psi_p$-process) with respect to a metric $d$ if
		\begin{equation*}
			\norm{X_t-X_{t'}}_{\psi_p}\leq d(t, t')\quad \text{for all }t, t'\in \cT.
		\end{equation*}
	\end{definition}
	
	According to the above definition, the size of the set $\cT$ may be infinite. To control the behavior of a $\psi_p$-process, we first need the notions of $\epsilon$-covering and covering number.
	
	\begin{definition}[Covering and covering number]
		A set $\cN_\err(\cT)$ is called an $\err$-covering on $(\cT,d)$ if for every $t\in \cT$, there exists $\pi(t) \in \cN_\err(\cT)$ such that $d(t,\pi(t)) \leq \err$. The covering number $\cN(\cT, d, \err)$ is defined as the smallest cardinality of an $\err$-covering for $(\cT,d)$:
		\begin{equation*}
		        \cN(\cT, d, \err) := \inf\{|\cN_\err(\cT)| : \cN_\err(\cT) \text{ is an } \err\text{-covering}\}.
		\end{equation*}
	\end{definition}
	
	\begin{theorem}[Concentration of Orlicz process, Theorem 5.36 in \cite{wainwright2019high}]
		\label{thm::concentration-Orlicz-process}
		Let $\{X_t, t \in \cT\}$ be a $\psi_p$-process with respect to the metric $d$. Then, there is a universal constant $C>0$ such that
		\begin{equation*}
			\bP\left(\sup_{t\in \cT}X_t\geq C\cJ_p(0, D)+\delta\right)\leq \exp({-{\delta^p}/{D^p}}).
		\end{equation*}
		Here $D=\sup_{t, t'\in \cT}d(t, t')$ is the diameter of the set $\cT$ and $\cJ_p(\delta, D)=\int_{\delta}^{D}\log^{1/p}\left(1+\cN(\cT, d, \err)\right)d\err$ is called the generalized Dudley entropy integral.
	\end{theorem}
	By providing tractable upper bounds on $D$ and $\cJ_p(\delta, D)$, one can control the supremum of a $\psi_p$-process. In our next section, we provide specific sub-classes of $\psi_p$-process for which these quantities can be controlled efficiently.
	
	\subsection{Concentration Bounds for Gaussian Processes over Grassmannian Manifold}\label{subsec_grassmanian}
	
	In our subsequent analysis, Gaussian processes defined over the Grassmannian manifold, which are a specific category of Orlicz processes, will play a pivotal role. To set the stage, we start with the definition of a Gaussian process. 
	\begin{definition}[Gaussian process]
		A random process $\{X_t\}_{t\in \cT}$ is called a Gaussian process if, for any finite subset $\cT_0 \subset \cT$, the random vector $\{X_t\}_{t\in \cT_0}$ has a normal distribution.
	\end{definition}
	Our next lemma characterizes the expectation of the supremum of a Gaussian process in terms of its covering number.
	\begin{theorem}[Sudakov's minoration inequality, Theorem 5.30 in \cite{wainwright2019high}]
		\label{thm::sudakov}
		Let $\{X_t\}_{t\in \cT}$ be a zero-mean Gaussian process. Then, for any $\err\geq 0$, we have
		\begin{equation*}
			\bE\left[\sup_{t\in \cT}X_t\right]\geq \frac{\err}{2}\sqrt{\log\cN(\cT, d, \err)}.
		\end{equation*}
		Here, the metric $d$ is defined as $d(t_1, t_2)=\left(\bE\left[(X_{t_1}-X_{t_2})^2\right]\right)^{1/2}$ for any $t_1, t_2\in \cT$.
	\end{theorem}
	\begin{lemma}[Concentration of Gaussian Process, Lemma 6.12 in \cite{van2014probability}]
		\label{lem::concentration-gaussian-process}
		Let $\{X_t\}_{t\in \cT}$ be a Gaussian process. Then $\sup_{t\in \cT} X_t$ is sub-Gaussian with sub-Gaussian norm $\sup_{t\in \cT}\norm{X_t}_{\psi_2}$.
	\end{lemma}
	
	Next, we provide the definition of Grassmanian manifold.
 \begin{equation}\label{eq_grassmanian}
     \cG(n,p) = \{P\in \bR^{n\times n}: P^\top \!=\! P, P^2\!=\!P, \rank(P) = p\}.
 \end{equation}
	Our next lemma characterizes the covering number of the Grassmannian manifold $\cG(n,p)$.
	
	\begin{lemma}[Covering number for Grassmannian manifold, Proposition 8 in \cite{szarek1982nets}]
		\label{lem::covering-number-orthogonal}
		For a Grassmannian manifold $\cG(n, p)$, define the metric $d(P_1, P_2)= \norm{P_1-P_2}_F$ for any $P_1,P_2\in\cG(n,p)$. Then, for any $0<\err<\sqrt{2\min\{p,n-p\}}$, we have
		\begin{equation*}
			\left(\frac{c_1\sqrt{p}}{\err}\right)^{p(n-p)}\leq\cN(\cG(n, p), d, \err)\leq \left(\frac{c_2\sqrt{p}}{\err}\right)^{p(n-p)},
		\end{equation*}
		for some constants $c_1, c_2>0$.
	\end{lemma}
	
	\begin{lemma}[Lower bound of Gaussian process over Grassmannian manifold]
		\label{lem::lower-bound-Gaussian-process}
		Consider the Gaussian process $\sup_{X\in \cG(n, p)}\inner{A}{X}$, where $A$ is a standard Gaussian matrix and $p\leq n/2$. Then, with probability at least $1-\exp({-\Omega(np)})$, we have 
		\begin{equation*}
			\sup_{X\in \cG(n, p)}\inner{A}{X}\geq \frac{c_3}{2}\sqrt{np^2},
		\end{equation*}
		for $c_3=\sqrt{\log(c_1)/8}$, where $c_1$ is the constant appeared in \Cref{lem::covering-number-orthogonal}.
	\end{lemma}
	\proof{Proof}
		We first use Sudakov's Minoration Inequality (\Cref{thm::sudakov}) to characterize the expectation of the supremum:
		\begin{equation*}
			\begin{aligned}
				\bE\left[\sup_{X\in \cG(n, p)}\inner{A}{X}\right]
                &\geq \sup_{0<\err} \frac{\err}{2}\sqrt{\log \cN\left(\cG(n, p), \norm{\cdot}_F, \err\right)}\\
				&\stackrel{(a)}{\geq} \sup_{0<\err\leq \sqrt{2p}} \frac{\err}{2} \sqrt{p(n-p)\log\left(\frac{c_1\sqrt{p}}{\err}\right)}\\
				&\stackrel{\text{set }\err=\sqrt{p}}{\geq} \frac{1}{2}\sqrt{(n-p)p^2\log(c_1)}\\
				&\geq c_3\sqrt{np^2}.
			\end{aligned}
		\end{equation*} 
		In (a), we invoked the lower bound on the covering number $\cN\left(\cG(n, p), \norm{\cdot}_F, \err\right)$ from \Cref{lem::covering-number-orthogonal}. 
		Next, we apply \Cref{lem::concentration-gaussian-process} to control the deviation of the supremum from its expectation. To this goal, we first note that for any $X\in \cG(n, p)$, we have $\norm{\inner{A}{X}}_{\psi_2}\leq\norm{X}_F=\sqrt{p}$. According to \Cref{lem::concentration-gaussian-process}, $\sup_{X\in \cG(n, p)}\inner{A}{X}$ is sub-Gaussian with norm $\sqrt{p}$. Due to \Cref{lem::concentration-orlicz-rv}, we have 
		\begin{equation*}
			\begin{aligned}
				\bP\left(\sup_{X\in \cG(n, p)}\inner{A}{X}-c_3\sqrt{np^2}\leq -\delta\right)
                &\leq \bP\left(\sup_{X\in \cG(n, p)}\inner{A}{X}-\bE\left[\sup_{X\in \cG(n, p)}\inner{A}{X}\right]\leq -\delta\right)\\
				&\leq \exp\left\{-\frac{\delta^2}{\norm{\sup_{X\in \cG(n, p)}\inner{A}{X}}_{\psi_2^2}}\right\}\\
				&\leq \exp\left\{-\delta^2/p\right\}.
			\end{aligned}
		\end{equation*}
		Upon setting $\delta=\frac{c_3}{2}\sqrt{np^2}$, we know that with probability at least $1-\exp({-\Omega(np)})$ 
		\begin{equation*}
			\sup_{X\in \cG(n, p)}\inner{A}{X}\geq \frac{c_3}{2}\sqrt{np^2},
		\end{equation*}
		which completes the proof. \hspace*{\fill}$\Box$\par
	\endproof
	Our next lemma provides another useful bound on the supremum of another random process over Grassmannian manifold.
	\begin{lemma}
		\label{lem::sub-exponential-process}
		Suppose that $\{A_i\}_{i=1}^m$ are i.i.d. standard Gaussian matrices and $\{{w}\}_{i=1}^m$ are non-negative weights. Then, we have with probability at least $1-\exp({-\Omega(np)})$ 
		\begin{equation*}
			\sup_{X\in \cG(n, p)} \frac{1}{m}\sum_{i=1}^{m}w_i\inner{A_i}{X}^2\leq \frac{cnp^2}{m}\sum_{i=1}^{m}w_i.
		\end{equation*}
		For some constant $c>0$.
	\end{lemma}
	\proof{Proof}
		We first verify that $\frac{1}{m}\sum_{i=1}^{m}w_i\inner{A_i}{X}^2$ is indeed a $\psi_1$-process. To see this, for arbitrary two matrices $X, X'\in \cG(n, p)$, we have 
		\begin{equation*}
			\begin{aligned}
				\norm{\frac{1}{m}\sum_{i=1}^{m}w_i\left(\inner{A_i}{X}^2-\inner{A_i}{X'}^2\right)}_{\psi_1}
                &\leq \frac{1}{n}\sum_{i=1}^{n}w_i\norm{\inner{A_i}{X}^2-\inner{A_i}{X'}^2}_{\psi_1}\\
				&\leq \frac{1}{m}\sum_{i=1}^{m}w_i\norm{\inner{A_i}{X+X'}\inner{A_i}{X-X'}}_{\psi_1}\\
				&\stackrel{(a)}{\leq}\frac{1}{m}\sum_{i=1}^{m}w_i\norm{\inner{A_i}{X+X'}}_{\psi_2}\norm{\inner{A_i}{X-X'}}_{\psi_2}\\
				&\leq \frac{1}{m}\left(\sum_{i=1}^{m}w_i\right)\norm{X+X'}_F\norm{X-X'}_F\\
				&\leq \frac{2\sqrt{p}}{m}\left(\sum_{i=1}^{m}w_i\right) \norm{X-X'}_F,
			\end{aligned}
		\end{equation*}
		where in $(a)$, we used \Cref{lem_mult}. Therefore, $\frac{1}{m}\sum_{i=1}^{m}w_i\inner{A_i}{X}^2$ is a $\psi_1$-process. On the other hand, we have $\bE\left[\frac{1}{m}\sum_{i=1}^{m}w_i\inner{A_i}{X}^2\right]=\frac{1}{m}\sum_{i=1}^{m}w_i$. Next, we apply \Cref{thm::concentration-Orlicz-process} to control the concentration of $\sup_{X\in \cG(n, p)} \frac{1}{m}\sum_{i=1}^{m}w_i\inner{A_i}{X}^2$. For notational simplicity, we denote $\beta=\frac{2\sqrt{p}}{m}\sum_{i=1}^{m}w_i$. We have 
		\begin{equation}
            \label{eq::111}
			\begin{aligned}
			    &\bP\bigg(\sup_{X\in \cG(n, p)} \frac{1}{m}\sum_{i=1}^{m}\left(w_i\inner{A_i}{X}^2-w_i\right) \geq C\beta\cJ_1(0, D)+\beta\delta\bigg)\leq \exp({-{\delta}/{D}}),
			\end{aligned}
		\end{equation}
		where the diameter satisfies $D=\sup_{X, X'\in \cG(n, p)}\norm{X-X'}_F\leq 2\sqrt{p}$ and the generalized Dudley's integral can be bounded as 
		\begin{equation*}
			\begin{aligned}
				\cJ_1(0, D)&=\int_0^D \log\left(1+\cN(\cG(n, p), \norm{\cdot}_F, \err)\right)d\err\\
				&\stackrel{(a)}{\leq} \int_0^{2\sqrt{p}} 2(n-p)p\log\left(\frac{c_1\sqrt{p}}{\err}\right)d\err\\
				&\leq \int_0^{2} 2(n-p)p^{3/2}\log\left(\frac{c_1}{\err}\right)d\err\\
				&\leq c_4 np^{3/2}.
			\end{aligned}
		\end{equation*}
		Here $c_4=2\int_0^{2} \log\left(\frac{c_1}{\err}\right)d\err$ is a universal constant. Then, upon choosing $\delta=c_5np^{3/2}$ with $c_5=Cc_4$ we have with probability at least $1-e^{-\Omega(np)}$ 
		\begin{equation*}
			\sup_{X\in \cG(n, p)} \frac{1}{m}\sum_{i=1}^{m}w_i\inner{A_i}{X}^2\leq \frac{2c_5np^2}{m}\sum_{i=1}^{m}w_i.
		\end{equation*}
		This completes the proof. \hspace*{\fill}$\Box$\par
	\endproof
	
	\subsection{Gaussian Comparison Inequalities}\label{subsec_other_lemma}
	In this subsection, we provide other useful lemmas comparing the extremum of two Gaussian processes.
 \begin{lemma}[Adapted from \protect{\cite[Corollary~3.13.]{ledoux1991probability}}]\label{lem_Gaussian_exp}
                        For two Gaussian processes $X_{i, j, k}, Y_{i, j, k}$ with $i\in I, j\in J, k\in K$, suppose that 
                        \begin{itemize}
                            \item $\bE[X_{i, j, k}^2]=\bE[Y_{i, j, k}^2]$, for all $i\in I, j\in J, k\in K$;
                            \item $\bE[X_{i, j, k}X_{i, l, k}]\leq\bE[Y_{i, j, k}Y_{i, l, k}]$, for all $i\in I, j, l\in J, k\in K$;
                            \item $\bE[X_{i, j, k}X_{l, m, k}]\geq\bE[Y_{i, j, k}Y_{l, m, k}]$, for all $i\neq l\in I, j, m\in J, k\in K$.
                        \end{itemize}
                        Then, we have 
                        \begin{equation*}
                            \bE\left[\min_{k\in K}\min_{i\in I}\max_{j\in J}Y_{i, j, k}\right]\leq \bE\left[\min_{k\in K}\min_{i\in I}\max_{j\in J}X_{i, j, k}\right].
                        \end{equation*}
    \end{lemma}

Our next lemma is a standard concentration bound on a Lipschitz function of a Gaussian random variable.
 \begin{lemma}[Proposition 5.34 in \cite{vershynin2010introduction}]\label{lem_lip}
                        Let $f$ be a real valued Lipschitz function on $\bR^n$ with Lipschitz constant $K$. Let $x$ be the standard normal random vector in $\bR^n$. Then for every $t\geq 0$ one has
                        \begin{equation*}
                            \bP\left(f(x)-\bE[f(x)]\leq -t\right)\leq \exp({-t^2/2K^2}).
                        \end{equation*}
    \end{lemma}
    Equipped with the above two lemmas, we present the following result, which will play an important role in our subsequent analysis. In what follows, a standard Gaussian matrix refers to a matrix whose entries are independently drawn from a standard Gaussian distribution.
    \begin{lemma}\label{lem_uniform_svd}
        Suppose that that integers $d_1,d_2,m,k$ satisfy $d_1\geq cd_2$ and $m\leq \frac{(\sqrt{c}-1)^2}{4c}\cdot d_1k$ for some $c>1$. Suppose that $A\in \bR^{m\times d_1d_2}$ is a standard Gaussian matrix. With probability at least $1-\exp(-\Omega(d_1k))$, we have
        \begin{equation*}
            \min_{V\in \cO_{d_2\times k}}\left\{\sigma_{\min}(A\cdot(V\otimes I_{d_1\times d_1}))\right\}\geq \frac{\sqrt{c}-1}{4\sqrt{c}}\sqrt{d_1k}.
        \end{equation*}
    \end{lemma}

    \proof{Proof}
        Let us denote $A_V = A\cdot(V\otimes I_{d_1\times d_1})$. By Min-max Theorem for singular values, we have
        \begin{equation*}
            \begin{aligned}
                &\sigma_{\min}(A_V) = \min_{u\in \bR^{m}, \norm{u}=1} \max_{W\in \bR^{d_1\times k}, \norm{W}_F=1}\inner{u}{A_V \vecc(W)}.
            \end{aligned}
        \end{equation*}
        Define $X_{u,W,V} = \inner{u}{A_V \vecc(W)}$ as a Gaussian process.
        Moreover, consider an auxiliary Gaussian process $Y_{u, W, V}=\frac{\sqrt{2}}{2}\left(\inner{g}{u}+\inner{GV}{W}\right)$ with $u\in \bR^{m}, V\in \bR^{d_2\times k}, W\in \bR^{d_1\times k}$, where $g\in \bR^{m}$ and $G\in \bR^{d_1\times d_2}$ are standard Gaussian vector and matrix. It is easy to verify that the Gaussian processes $X_{u,W,V}$ and $Y_{u,W,V}$ satisfy the conditions of \Cref{lem_Gaussian_exp}. Therefore, we have 
        \begin{equation*}
            \begin{aligned}
                \bE\left[\min_{V\in \cO_{d_2\times k}}\sigma_{\min}(A_V)\right]
            &= \bE\left[\min_{V\in \cO_{d_2\times k}}\min_{u\in \bR^{m}, \norm{u}=1} \max_{W\in \bR^{d_1\times k}, \norm{W}_F=1}X_{u,W,V} \right]\\
            &\geq \bE\left[\min_{V\in \cO_{d_2\times k}}\min_{u\in \bR^{m}, \norm{u}=1} \max_{W\in \bR^{d_1\times k}, \norm{W}_F=1}Y_{u,W,V} \right]\\
            & = \bE\left[\min_{V\in \cO_{d_2\times k}}\max_{W\in \bR^{d_1\times k},\norm{W}_F=1}\inner{GV}{W}\right]+\bE\left[\min_{u\in \bR^m, \norm{u}=1}\inner{g}{u}\right]\\
                            &\geq \bE\left[\min_{V\in \cO_{d_2\times k}}\norm{GV}_F\right]-\bE\left[\norm{g}\right]\\
                            &\geq \bE\left[\sqrt{k}\sigma_{\min}(G)\right]-\bE\left[\norm{g}\right]\\
                            &\geq \sqrt{k}\left(\sqrt{d_1}-\sqrt{d_2}\right)-\sqrt{m}\\
                            &\geq \frac{\sqrt{c}-1}{2\sqrt{c}}\sqrt{d_1k}.
            \end{aligned}
        \end{equation*}
        Let us define $f(A) = \min_{V\in \cO_{d_2\times k}}\sigma_{\min}(A_V)$. It is easy to verify that $|f(A)-f(A')|\leq \norm{A-A'}_F$, for every $A,A'\in \bR^{m\times d_1d_2}$. This implies that $f(A)$ is Lipschitz with constant 1. Therefore, \Cref{lem_lip} can be applied:
        \begin{equation*}
            \begin{aligned}
                &\bP(f(A)-\bE[f(A)]\leq -t)\leq \exp(-t^2/2)\implies \bP\left(f(A)\geq \frac{\sqrt{c}-1}{2\sqrt{c}}\sqrt{d_1k}-t\right)\geq 1-\exp(-t^2/2).
            \end{aligned}
        \end{equation*}

        Substituting $t = \frac{\sqrt{c}-1}{4\sqrt{c}}\sqrt{d_1k}$ in the above inequality completes the proof. \hspace*{\fill}$\Box$\par 
    \endproof

 \subsection{Other Useful Lemmas}\label{subsec_useful}
	\begin{lemma}\label{lem_Gaussian}
		Suppose that $A_1, \cdots, A_m$ are i.i.d. standard Gaussian matrices. For any sequence of scalars $w_1, \cdots, w_m$, the matrix $B=(\sum_{i=1}^m w_i^2)^{-1/2}\sum_{i=1}^mw_i A_i$ is a standard Gaussian matrix provided that $\sum_{i=1}^m w_i^2>0$.
	\end{lemma}
 \proof{Proof}
     It is easy to see that each element of $\sum_{i=1}^mw_i A_i$ is i.i.d. and Gaussian with mean zero and variance $\sum_{i=1}^m w_i^2$. Therefore, the elements of $B$ are all i.i.d. with standard Gaussian distribution. \hspace*{\fill}$\Box$\par
 \endproof
	
	\begin{lemma}
		\label{lem::weighted-sub-Gaussian}
		Suppose that $A_1, \cdots, A_m$ are i.i.d. standard Gaussian matrices and $w_1, \cdots, w_m$ are nonnegative scalars with $\sum_{i=1}^m w_i^2>0$. Then, for any $\delta>0$ and a fixed matrix $X\in \bR^{n\times n}$ with $\norm{X}_F=1$, we have 
		\begin{equation*}
			\begin{aligned}
			    &\bP\left(\left|\frac{1}{m}\sum_{i=1}^{m}w_i|\inner{A_i}{X}|-\sqrt{\frac{2}{\pi}}\frac{1}{m}\sum_{i=1}^{m}w_i\right|\geq \delta\right)\leq 2\exp\left({-\frac{m^2\delta^2}{C\sum_{i=1}^{m}w_i^2}}\right).
			\end{aligned}
		\end{equation*}
	\end{lemma}
 \begin{sloppypar}
     \proof{Proof}
		Note that $\inner{A_i}{X}\stackrel{iid}{\sim} \cN(0, 1)$ and $\bE[|\inner{A_i}{X}|] = \sqrt{\frac{2}{\pi}}$. Due to \Cref{lem::centering}, we have $\norm{|\inner{A_i}{X}|-\bE[|\inner{A_i}{X}|]}_{\psi_2}\leq\norm{|\inner{A_i}{X}|}_{\psi_2}$. Then final result follows from \Cref{lem::concentration-sum-independent}. \hspace*{\fill}$\Box$\par
	\endproof
 \end{sloppypar}
	{
	\begin{lemma}\label{lem_max_Gaussian}
		Suppose that $A_1, \cdots, A_m$ are i.i.d. standard Gaussian matrices of the shape $d_1\times d_2$. We have
		\begin{equation*}
		    \!\bP\!\left(\!\max_{1\leq i\leq m}\!\norm{A_i}_F\geq \sqrt{2d_1d_2}\right)\!\leq \exp\left({\log m\!-\!\Omega(d_1d_2)}\right).
		\end{equation*}
	\end{lemma}
	\proof{Proof}
		For any fixed $i$ and $1\leq {\alpha}\leq d_1, 1\leq {\beta}\leq d_2$, $(A_i)[{\alpha, \beta}]^2-1$ is a sub-exponential random variable with zero mean and parameter $1$. Therefore, an application of \Cref{lem::concentration-sum-independent} implies that $\bP(\norm{A_i}_F\geq \sqrt{1+\delta}{\sqrt{d_1d_2}})\leq \exp(-\Omega({d_1d_2}))$. Setting $\delta =1$ followed by a union bound leads to the final result. \hspace*{\fill}$\Box$\par
	\endproof
	
    \begin{lemma}[Corollary 5.35 in~\cite{vershynin2010introduction}]\label{lem_gaussian}
		Consider a random matrix $X\in\bR^{N\times n}$ with i.i.d. entries drawn from $\cN(0,1)$ and $N>n$. Then for every $t> 0$, with probability at least $1 - 2 \exp\{-t^2/2\}$ one has  
		\begin{equation}\nonumber
			\sigma_{\min}(X)\geq \sqrt{N}-\sqrt{n}-t.
		\end{equation}
	\end{lemma}
    }
	
	\begin{lemma}[Concentration of binomial random variable, \cite{chung2006concentration}]
		\label{lem::concentration-Bernoulli}
		Suppose that $X$ has a binomial distribution with parameters $(m,p)$. Then, we have 
		\begin{equation*}
			\bP\left(\left|X-pm\right|\geq \delta\right)\leq 2\exp({-{\delta^2}/{2pm}}).
		\end{equation*}
		In particular, when choosing $\delta =\frac{pm}{2}$, we have that with probability at least $1-2\exp({pm/4})$,
		\begin{equation*}
			X\geq \frac{pm}{2}, \quad \text{and}\quad X\leq \frac{3pm}{2}.
		\end{equation*}
	\end{lemma}
	
	\begin{lemma}\label{lem_bernoulli_lb}
		Suppose that $X$ has a binomial distribution with parameters $(m,p)$ and $0\leq p<1/2$. Then, with probability at least $1-\exp(-\Omega((1-2p)^2m))$, we have
		\begin{equation*}
			X\leq \max\left\{\frac{3}{8}, \frac{1-(1-2p)^2}{2}\right\}\cdot m.
		\end{equation*}
	\end{lemma}
	\proof{Proof}
		We consider two cases. First, suppose that $1/4\leq p<1/2$. Define $\eta = 1-2p$. According to \Cref{lem::concentration-Bernoulli}, we have 
		\begin{equation*}
			\begin{aligned}
			    &\bP\left(X\!\leq pm(1+\eta)\right)\geq 1-\exp(-\Omega(\eta^2m))
			\implies \bP\!\left(\!X\leq \frac{1-(1-2p)^2}{2}{m}\!\!\right)\!\geq 1\!-\!\exp(-\Omega((1\!-\!2p)^2m)),
			\end{aligned}
		\end{equation*}
		where the last inequality follows from the fact that $pm(1+\eta) = \frac{1-(1-2p)^2}{2}{m}$.
		On the other hand, for $0\leq p\leq 1/4$, \Cref{lem::concentration-Bernoulli} implies that
		\begin{equation*}
			\bP\left(X\leq 3m/8\right)\geq 1-\exp(-\Omega(m)).
		\end{equation*}
		Combining the above two cases leads to the desired result. \hspace*{\fill}$\Box$\par
	\endproof

 Armed with the concentration bounds discussed above, we are now prepared to present our main proofs.
 
\section{Symmetric Case: Sub-optimality via Parametric Second-order Perturbations}
\label{sec_sym}
At the core of our analysis lies a class of parametric second-order perturbations that can be used to show the sub-optimality of the true solutions for both symmetric matrix sensing and symmetric matrix completion. Let $X^\star = V^\star\Sigma^\star {V^\star}^\top$ be the eigen-decomposition of $X^\star$, where $V^\star\in\bR^{d\times r}$ is an orthonormal matrix and $\Sigma^\star\in\bR^{r\times r}$ is a diagonal matrix collecting the nonzero eigenvalues of $X^\star$. Recall that $\cO_{r\times k}:=\{R\in\bR^{r\times k}: RR^\top = I_{r\times r}\}$. The following lemma provides another characterization of the set $\cW = \{W\in \bR^{d\times k}: WW^\top = X^\star\}$.
	\begin{lemma}\label{lem_rotation}
		We have
		$$
		W\in \cW\qquad\iff \qquad W = V^*{\Sigma^\star}^{1/2}R\ \text{ for some $R\in \cO_{r\times k}$}.
		$$
	\end{lemma}
	\proof{Proof}
		If $W = V^*{\Sigma^\star}^{1/2}R$ for some $R\in \cO_{r\times k}$, then $WW^\top = X^\star$ and $W\in \cW$. Now, suppose that $W\in \cW$. We have
		\begin{equation*}
		    \begin{aligned}
			WW^\top = V^\star\Sigma^\star {V^\star}^\top
            &\implies {\Sigma^\star}^{-1/2}{V^\star}^\top WW^\top V^\star{\Sigma^\star}^{-1/2} = I_{r\times r}\nonumber\\
			&\implies {\Sigma^\star}^{-1/2}{V^\star}^\top W = R \ \text{for some $R\in \cO_{r\times k}$}\nonumber\\
			&\implies V^\star{V^\star}^\top W = V^\star{\Sigma^\star}^{1/2}R\ \text{for some $R\in \cO_{r\times k}$}.\nonumber
		\end{aligned}
		\end{equation*}
		Moreover, let $V^\star_{\perp}$ be the orthogonal complement of $V^\star$. We have
		\begin{equation*}
		    \begin{aligned}
		        WW^\top = V^\star\Sigma^\star {V^\star}^\top
            &\implies {{V^\star_{\perp}}^{\top}} WW^\top V^\star_{\perp}= 0_{(d-r)\times (d-r)}\implies {V^\star_{\perp}}^{\top} W = 0\implies  {V^\star_{\perp}}{V^\star_{\perp}}^{\top} W = 0.
		    \end{aligned}
		\end{equation*}
		Combining the above two equalities, we have
		\begin{equation*}
			WW^\top = V^\star{V^\star}^\top W + {V^\star_{\perp}}{V^\star_{\perp}}^{\top} W = V^\star{\Sigma^\star}^{1/2}R\nonumber
		\end{equation*}
		for some $R\in \cO_{r\times k}$. This completes the proof. \hspace*{\fill}$\Box$\par
	\endproof
	Based on the above lemma, the set of true solutions can be characterized as $\cW = \{V^\star{\Sigma^\star}^{1/2}R: R\in\cO_{r\times k}\}$. This characterization of $\cW$ will be useful in our subsequent analysis. Our goal is to show that, {for every $\trueW\in\cW$, there exists a descent direction $\Delta W\in \cB$ such that $f_{\ell_1}(\trueW+\gamma\Delta W) - f_{\ell_1}(\trueW)\leq -\Omega(\gamma^2)$ for every $0<\gamma\leq \bar{\gamma}$.}
	{To this end, for any $\trueW\in \cW$,} we consider the following set of \textit{parametric second-order perturbations}:
	{\begin{equation*}
		\begin{aligned}
			&\widehat{\cU}_{W^\star} := \left\{U R: U\in \bR^{d\times (k-r)}, \norm{U}_F\leq 1,R\in \cO_{(k-r)\times k}, \trueW {R}^\top = 0\right\}.
		\end{aligned}
	\end{equation*}}
	{Note that any perturbation $\Delta W = U R\in \widehat{\cU}_{W^\star}$} is indeed second-order since:
	\begin{equation*}
	    \begin{aligned}
	        &\norm{(\trueW+{\gamma}\Delta W)(\trueW+{\gamma}\Delta W)^\top - \trueW{\trueW}^\top}_F= {\gamma^2}\norm{UU^\top}_F\leq \gamma^2.
	    \end{aligned}
	\end{equation*}
	Evidently, we have {$\widehat{\cU}_{W^\star}\subset \cB$}. Moreover, {$\widehat{\cU}_{W^\star}$} is non-empty since, in light of $\cW = \{V^\star{\Sigma^\star}^{1/2}R: R\in\cO_{r\times k}\}$, there always exists $R\in \cO_{(k-r)\times k}$ such that $\trueW {R}^\top = 0$. {Therefore, to prove our main result, it suffices to find $\bar U\in \cB^{k-r} := \{U: U\in \bR^{d\times (k-r)}, \norm{U}_F\leq 1\}$ such that the direction defined as $\Delta W = \bar UR\in \widehat{\cU}_{W^\star}$ achieves the desired decrease in the objective function. To this goal, note that, for every $\Delta W = UR\in \widehat{\cU}_{W^\star}$, we have}
    {\begin{align}
 \label{eq_ub_uniform}
		f_{\ell_1}(\trueW+\gamma\Delta W) - f_{\ell_1}(\trueW)
         = \frac{1}{m}\sum_{i=1}^m\left|\inner{A_i}{\gamma^2 UU^\top}-\epsilon_i\right| - |\epsilon_i|.
	\end{align}
    } {An important property of the above equation is that its right-hand side is entirely independent of the specific choice of $\trueW$. In other words, proving the existence of a {\it single} $\bar{U}$ that achieves a second-order decrease for the right-hand side suffices to ensure that the parametric perturbation $\Delta W = \bar U R \in \widehat{\cU}_{W^\star}$---which now depends on $W^\star$---achieves the desired decrease for {\it all} $W^\star\in \cW$. The rest of our analysis is devoted to establishing the existence of such $\bar U$ for the symmetric matrix completion and symmetric matrix sensing. In fact, this can be easily done for the symmetric matrix completion, as shown below.
    } 
 
\subsection{Proof of \Cref{thm_completion_sym}.} 
 According to our assumption, there exists $({\alpha, \alpha})\in {\Omega}$ such that $E[{\alpha, \alpha}]\geq t_0$. Define $\bar{U}\in \bR^{d\times (k-r)}$ as  
	\begin{equation*}
	    \bar{U}[{\alpha,\beta}] = \begin{cases}
		{1} & \text{if}\ \ ({\alpha,\beta}) = ({\alpha}, 1)\\
		0 & \text{otherwise}
	\end{cases}.
	\end{equation*}
	Therefore,
    {\begin{equation}
        \label{eq_x1}
	\begin{aligned}
        \frac{1}{m}\sum_{i=1}^m\left|\inner{A_i}{\gamma^2\bar U\bar U^\top}-\epsilon_i\right| - |\epsilon_i|
		&= \frac{1}{m}\left(\left|E[{\alpha, \alpha}] - \left(\gamma^2\bar U\bar U^\top\right)[\alpha, \alpha]\right| - |E[\alpha, \alpha]|\right)\\
		&= -\frac{1}{m}\gamma^2\\
        &\leq -\frac{1}{2sd^2}\gamma^2.
	\end{aligned}
    \end{equation}}
	Here the second to last equality follows from the assumption $\gamma\leq \sqrt{t_0}$ which implies $E[{\alpha, \alpha}] - \left(\bar U\bar U^\top\right)[{\alpha, \alpha}]\geq 0$, and the last inequality follows from the fact that $m$ is binomial with parameters $(s,d^2)$ and hence, according to \Cref{lem::concentration-Bernoulli}, $m\leq 2sd^2$ with probability at least $1-\exp(-\Omega(sd^2))$. This inequality combined with \Cref{eq_ub_uniform} completes the proof. $\hfill\square$
 \subsection{Proof of \Cref{thm_no_benign_sym}}

 Next, we extend our analysis to the symmetric matrix sensing. 
 When the measurement matrices do not follow the matrix completion model, the explicit perturbation defined in the proof of \Cref{thm_completion_sym} may no longer lead to a decrease in the objective value. To address this issue, we provide a more delicate upper bound for \Cref{eq_ub_uniform}. Recall that $\cS$ is the index set of the noisy measurements, and define $\bar\cS = [m]\backslash \cS$ as the set of \textit{clean} measurements. Moreover, define $\cS_t = \{i: |\epsilon_i|\geq t_0\}$. The set $\cS_t$ denotes the set of measurements for which the magnitude of noise exceeds $t_0$.
	\begin{lemma}\label{lem_ub_ms_sym}
		Suppose that $\gamma\leq \sqrt{\frac{t_0}{\max_i\{\norm{A_i}_F\}}}$. For every $U\in\cB^{k-r}$, we have
		{\begin{equation*}
		    \begin{aligned}
		        &\frac{1}{m}\sum_{i=1}^m\left|\inner{A_i}{\gamma^2UU^\top}-\epsilon_i\right| - |\epsilon_i|\leq \gamma^2\cdot\left(A\left(UU^\top\right)+B\left(UU^\top\right)\right),
		    \end{aligned}
		\end{equation*}}

		where 
		$$A\left(UU^\top\right) = \frac{1}{m}\sum_{i\in\cS_t}- \sign(\epsilon_i)\inner{A_i}{UU^\top},\qquad B\left(UU^\top\right)= \frac{1}{m}\sum_{i\not\in \cS_t}\left|\inner{A_i}{UU^\top}\right|.
		$$
	\end{lemma}
	\proof{Proof}
		We have
		\begin{equation*}
		    \begin{aligned}
		        &\frac{1}{m}\sum_{i=1}^m\left|\inner{A_i}{{\gamma^2}UU^\top}-\epsilon_i\right| - |\epsilon_i| = \underbrace{\frac{1}{m}\sum_{i\in \cS_t}\left|\inner{A_i}{{\gamma^2}UU^\top}-\epsilon_i\right| - |\epsilon_i|}_{E_1}+\underbrace{\frac{1}{m}\sum_{i\not\in \cS_t}\left|\inner{A_i}{{\gamma^2}UU^\top}-\epsilon_i\right| - |\epsilon_i|}_{E_2}.
		    \end{aligned}
		\end{equation*}
		To bound $E_2$, we apply triangle inequality, which leads to
		\begin{equation*}
		    \begin{aligned}
		        &\frac{1}{m}\sum_{i\not\in \cS_t}\left|\inner{A_i}{{\gamma^2}UU^\top}-\epsilon_i\right| - |\epsilon_i|\leq \frac{1}{m}\sum_{i\not\in \cS_t}\left|\inner{A_i}{{\gamma^2}UU^\top}\right|={\gamma^2\cdot}B\left(UU^{\top}\right).
		    \end{aligned}
		\end{equation*}
		Next, we provide an upper bound for $E_1$. To this goal, we write, for every $i\in \cS_t$,
		\begin{equation*}
			\left|\inner{A_i}{{\gamma^2}UU^\top}-\epsilon_i\right|-|\epsilon_i| 
			=- {\gamma^2}\sign(\epsilon_i)\inner{A_i}{UU^\top},
		\end{equation*}
		where the equality follows from $\left|\inner{A_i}{{\gamma^2}UU^\top}\right|\leq {\gamma^2}\norm{A_i}_F\norm{UU^{\top}}_F\leq \max_{i}\{\norm{A_i}_F\}\gamma^2\leq t_0$. This implies that $E_1= {\gamma^2\cdot}A\left(UU^\top\right)$, thereby completing the proof. \hspace*{\fill}$\Box$\par
	\endproof

    {We note that the above lemma deterministically holds for all \( U \in \mathcal{B}^{k-r} \). Our next goal is to establish the existence of at least one \( U \) for which \( A(UU^\top) + B(UU^\top) \leq -\Omega(1) \) with high probability. This result, combined with \eqref{eq_ub_uniform}, suffices to prove Theorem~\ref{thm_no_benign_sym}.} Note that $B\left(UU^\top\right)\geq 0$ for every $U\in {\cB^{k-r}}$, so our hope is to show that $A\left(UU^\top\right)$ can take sufficiently negative value to dominate $B\left(UU^\top\right)$. As will be shown in our next lemma, this can be established with a high probability. In particular, we show that, when the measurement matrices follow the matrix sensing model, there exists, with high probability, $U\in{\cB^{k-r}}$ and coefficients $C_A>C_B\geq 0$ such that
	{\begin{equation}\label{eq_ABC}
		A\left({U}{U}^\top\right)\leq -C_A,\quad B\left({U}{U}^\top\right)\leq C_B,
	\end{equation}}
	\begin{lemma}\label{lem_ABC2}
		Define $r_0 = \min\{k-r,d/2\}$. Suppose that the measurement matrices follow the matrix sensing model in  \Cref{asp_sensing}. Conditioned on the noise $\{\epsilon_i\}_{i=1}^m$ and with probability at least $1-\exp(-\Omega(dr_0))$ over the randomness of the measurement matrices, there exists $U\in{\cB^{k-r}}$ such that \Cref{eq_ABC} holds with
		\begin{equation*}
			C_A = c\sqrt{\frac{dr_0|\cS_t|}{m^2}},\qquad C_B =  1,
		\end{equation*}
		for some constant $c>0$.
	\end{lemma}
	The proof makes extensive use of the concentration bounds for Gaussian processes over the Grassmanina manifold, as discussed in Section~\ref{subsec_grassmanian}. 
 Before presenting the proof of \Cref{lem_ABC2}, we show its application in the proof of \Cref{thm_no_benign_sym}.\vspace{1mm}
	
	\noindent{\it Proof of \Cref{thm_no_benign_sym}.}
	{For every $\trueW\in \cW$, there exists $U\in \cB^{k-r}$ such that} 
	\begin{equation*}
	    \begin{aligned}
	         \frac{1}{m}\sum_{i=1}^m\left|\inner{A_i}{{\gamma^2}UU^\top}-\epsilon_i\right| - |\epsilon_i|
		&\stackrel{(a)}{\leq} {\gamma^2\cdot}\left(A\left(UU^\top\right)+B\left(UU^\top\right)\right)\\
		&\stackrel{(b)}{\leq} -(C_A-C_B)\gamma^2\\
		&\stackrel{(c)}{\leq} -\left(c\sqrt{\frac{dr_0|\cS_t|}{m^2}}-1\right)\gamma^2\\
		&\stackrel{(d)}{\leq} -\left(c\sqrt{\frac{dr_0pp_0}{2m}}-1\right)\gamma^2\\
		&\leq -\frac{c}{2}\sqrt{\frac{dr_0pp_0}{2m}}\gamma^2
	    \end{aligned}
	\end{equation*}
	with probability at least $1-\exp(-\Omega(dr_0))-\exp(-\Omega(pp_0m))$, provided that 
	\begin{equation*}
	    \gamma^2\leq \frac{t_0}{\max_i\{\norm{A_i}_F\}}, \quad \text{and}\ \ m\leq  \frac{c^2}{8}\cdot{pp_0dr_0}.
	\end{equation*}
	In the above inequality, 
 $(a)$ follows from \Cref{lem_ub_ms_sym}, 
 $(b)$ and $(c)$ follow from \Cref{lem_ABC2}, and $(d)$ follows from the concentration of the binomial random variables (\Cref{lem::concentration-Bernoulli}), which implies
 $|\cS_t|\geq pp_0m/2$, with probability at least $1-\exp(-\Omega(pp_0m))$.
 On the other hand, \Cref{lem_max_Gaussian} implies that $\max_i\{\norm{A_i}_F\}\leq \sqrt{2}d$ with probability at least $1-\exp(-\Omega(d^2))$. This implies that 
	\begin{equation*}
	    \frac{t_0}{\max_i\{\norm{A_i}_F\}}\geq \frac{t_0}{\sqrt{2d}}\geq \gamma^2
	\end{equation*}
	with probability at least $1-\exp(-\Omega(d^2))$, where the last inequality follows from our assumed upper bound on $\gamma^2$. {For this choice of $U$, the direction $\Delta W = UR\in \widehat{\cU}_{W^\star}$ achieves the desired result in light of \Cref{eq_ub_uniform}, thereby completing the proof.}
	$\hfill\square$\vspace{2mm}

Next, we present the proof of \Cref{lem_ABC2}.

 \noindent{\it Proof of \Cref{lem_ABC2}.}
 Let $\cC = ({{1}}/{\sqrt{r_0}})\cG(d,r_0)$, where $\cG(d,r_0)$ is the Grassmanian manifold defined as \Cref{eq_grassmanian} and $r_0 = \min\{k-r,d/2\}$. It is easy to see that if $X\in \cC$, then $X = UU^\top$ for some $U\in {\cB^{k-r}}$. Consider $\bar X = \arg\min_{X\in\cC}A(X)$ {and let $\bar X = \bar U\bar U^\top$ with $\bar U\in {\cB^{k-r}}$. Note that the existence and attainability of $\bar X$ are guaranteed due to the compactness of $\cC$ \cite{guillemin2010differential}.} We prove the desired inequality is attained at $\bar U$. 
	
	\paragraph{Bounding $B(\bar X)$:} According its definition, $\bar X$ is independent of $\{A_i\}_{i\not\in\cS_t}$. Let $\{B_i\}_{i\in \cS_t}$ be independent copies of the measurement matrices that are generated according to the matrix sensing model ( \Cref{asp_sensing}). We have
	\begin{equation*}
	    \begin{aligned}
	        B(\bar X) &= \frac{1}{m}\sum_{i\not\in \cS_t}\left|\inner{A_i}{\bar X}\right|\leq \frac{1}{m}\sum_{i\not\in \cS_t}|\inner{A_i}{\bar X}|+\frac{1}{m}\sum_{i\in \cS_t}|\inner{B_i}{\bar X}|\leq {1}
	    \end{aligned}
	\end{equation*}
	with probablity at least $1-\exp(-\Omega(m))$, where in the last inequality we invoked \Cref{lem::weighted-sub-Gaussian} with $\delta = 1-\sqrt{\frac{2}{\pi}}$ and $w_i=1$.
	
	\paragraph{Bounding $A(\bar X)$.} Recall that { we choose $\bar X = \arg\min_{X\in\cC}A(X)$, or equivalently,}
	\begin{equation*}
		A(\bar X) = \inf_{X\in \cC}\left\{\frac{1}{m}\sum_{i\in\cS_t}-\sign(\epsilon_i)\inner{A_i}{X}\right\}.
	\end{equation*}
	Note that, unlike $B(\bar X)$, the choice of $\bar X$ depends on the measurement matrices involved in $A(\bar X)$. To control this dependency, consider $\tilde A =|\cS_t|^{-1/2}\sum_{i\in \cS_t}\sign(\epsilon_i)A_i$. Invoking \Cref{lem_Gaussian} with $\omega_i =\sign(\epsilon_i)$, we know that $\tilde A$ is a standard Gaussian matrix and hence
	\begin{equation*}
	    \begin{aligned}
	        A(\bar X) &= \frac{\sqrt{|\cS_t|}}{m}\inf_{X\in\cC}\inner{-\tilde A}{X}\\
        &= -\frac{\sqrt{|\cS_t|}}{m}\sup_{X\in\cC}\inner{\tilde A}{X}\\
        &= -\frac{\sqrt{|\cS_t|}}{m}\frac{{1}}{\sqrt{r_0}}\sup_{X\in\cG(d,r_0)}\inner{\tilde A}{X}\\
        &\leq -c\sqrt{\frac{dr_0|\cS_t|}{m^2}},
	    \end{aligned}
	\end{equation*}
	with probability at least $1-\exp{-\Omega(dr_0)}$, where the last inequality follows from \Cref{lem::lower-bound-Gaussian-process}.
	This completes the proof of this lemma.$\hfill\square$
	
	\section{Asymmetric Case: Non-criticality via Parametric First-order Perturbations}\label{seC_Asym}
 
Next, we extend our analysis to the asymmetric setting. Before delving into the details, we first highlight two key challenges that distinguish the asymmetric setting from its symmetric counterpart. Firstly, \Cref{lem_rotation} no longer holds for the asymmetric setting, since the set of true solutions $\cW$ cannot be fully characterized via the set of orthonormal matrices. Consequently, the set of second-order perturbations proposed for the symmetric setting cannot be used to show the sub-optimality of {all} true solutions in the asymmetric setting. Secondly, the sub-optimality of a true solution does not automatically imply its non-criticality. 

 In this section, we overcome these challenges. 
At the core of our results lies a set of parametric {\it first-order} perturbations that can be used to prove the non-criticality of the true solutions for both asymmetric matrix sensing and asymmetric matrix completion. 
{Unlike the symmetric case, where true solutions are explicitly constructed, we take a different approach in the asymmetric case, avoiding reliance on such constructions and instead directly establishing non-criticality through carefully designed perturbations.}
	To this goal, we will rely on Lemma~\ref{lem_critical_mr}, which implies that, in order to show the non-criticality of a point $\fW^\star$, it suffices to obtain a perturbation $\Delta \fW = (\Delta W_1, \Delta W_2)$ such that $f'_{\ell_1}(\fW^\star, \Delta\fW)<0$. {To this goal, we show that, with high probability, there exists $\Delta\fW\in {\cB}$ such that $f_{\ell_1}(\fW^\star+{\gamma}\Delta \fW) - f_{\ell_1}(\fW^\star) = -\Omega(\gamma)$ for every $0\leq \gamma\leq \gamma_0$.} As will be shown later, this result automatically implies the non-criticality of $\fW^\star$. Recall that $\bar\cS$ and $\cS$ are the sets of clean and noisy measurements, respectively. Moreover, recall that $\cS_t = \{i: |\epsilon_i|\geq t_0\}$.  Given any $\fW^\star\in \cW$, one can write
	\begin{equation*}
	    \begin{aligned}
	        &f_{\ell_1}(\fW^\star+{\gamma}\Delta\fW)-f_{\ell_1}(\fW^\star) \\
        &= \frac{1}{m}\sum_{i\in\bar{\cS}}|\langle A_i, (\trueW_1\!+\!{\gamma}\Delta W_1)(\trueW_2\!+\!{\gamma}\Delta W_2)-\trueW_1 \trueW_2\rangle|\nonumber\\
		&\quad+\frac{1}{m}\sum_{i\in{\cS}}\left(|\langle A_i, (\trueW_1+{\gamma}\Delta W_1)(\trueW_2+{\gamma}\Delta W_2)-\trueW_1 \trueW_2\rangle-\epsilon_i|-|\epsilon_i|\right)\\
		&= \frac{1}{m}\sum_{i\in\bar{\cS}}|\langle A_i, {\gamma}\Delta W_1 \trueW_2+{\gamma}\trueW_1\Delta W_2+{\gamma^2}\Delta W_1 \Delta W_2\rangle|\nonumber\\
		&\quad+\frac{1}{m}\sum_{i\in{\cS}}\left(|\langle A_i,{\gamma}\Delta W_1 \trueW_2\!+\!{\gamma}\trueW_1\Delta W_2\!+\!{\gamma^2}\Delta W_1 \Delta W_2\rangle-\epsilon_i|-|\epsilon_i|\right).
	    \end{aligned}
	\end{equation*}
	
	To analyze the infimum of the loss difference over $\Delta \fW\in {\cB}$, we consider the following set of \textit{parametric first-order perturbations} at $\fW^\star\in \cW$:
	\begin{equation}\label{eq_cU}
		\begin{aligned}
			{\cU_{\fW^\star}(\gamma_0)} := \big\{ (U_1, 0_{k\times d_2}): &
			\ U_1\in \bR^{d_1\times k}, \norm{U_1}_F\leq {1},\\
			&\ \langle A_i, U_1 W^\star_2 \rangle = 0, &&\forall i\in \bar{\cS},\\
			&\ \langle A_i, U_1 W^\star_2 \rangle \epsilon_i \geq 0, &&\forall i\in {\cS},\\
			&\ {\gamma_0}\left|\langle A_i, U_1 W^\star_2\rangle\right| \leq |\epsilon_i|, &&\forall i\in {\cS}\big\}.
		\end{aligned}
	\end{equation}
 \begin{sloppypar}
     To streamline the presentation, we assume without loss of generality that $d_1\geq d_2$ throughout this section.
	Evidently, ${\cU_{\fW^\star}(\gamma_0)}$ is non-empty since $(0_{d_1\times k}, 0_{k\times d_2})\in {\cU_{\fW^\star}(\gamma_0)}$, and we have ${\cU_{\fW^\star}(\gamma_0)}\subseteq {\cB}$. The following lemma provides a more tractable upper bound on $f_{\ell_1}(\fW^\star+{\gamma}\Delta \fW) - f_{\ell_1}(\fW)$ when the direction $\Delta\fW$ is restricted to ${\cU_{\fW^\star}(\gamma_0)}$. 
 \end{sloppypar}
	\begin{proposition}\label{prop_BFU}
		{Given any $\fW^\star\in \cW$, $\gamma_0>0$, and $\Delta\fW\in\cU_{\fW^\star}(\gamma_0)$, we have 
		\begin{equation}\label{eq_ub_cU}
			\begin{aligned}
			    &f_{\ell_1}(\fW^\star+{\gamma}\Delta \fW) - f_{\ell_1}(\fW)\leq -{\gamma\cdot}\left(\frac{1}{m}\sum_{i\in{\cS_t}}|\langle A_i, \Delta W_1 \trueW_2\rangle|\right), \quad \text{for every $0\leq \gamma\leq \gamma_0$}.
			\end{aligned}
		\end{equation}}
	\end{proposition}
	\proof{Proof}
		One can write
		{\begin{equation*}
		    \begin{aligned}
		        f_{\ell_1}(\fW^\star+{\gamma}\Delta \fW) - f_{\ell_1}(\fW)
            &\stackrel{(a)}{=} -\left(\frac{1}{m}\sum_{i\in{\cS}}|\epsilon_i| \!-\! |\langle A_i, {\gamma}\Delta W_1 \trueW_2\rangle-\epsilon_i|\!\right)\nonumber\\
			&\stackrel{(b)}{=} -{\gamma\cdot}\left( \frac{1}{m}\sum_{i\in{\cS}}|\langle A_i, \Delta W_1 \trueW_2\rangle|\right)\nonumber\\
			&\leq -{\gamma\cdot}\left( \frac{1}{m}\sum_{i\in{\cS_t}}|\langle A_i, \Delta W_1 \trueW_2\rangle|\right)
		    \end{aligned}
		\end{equation*}}
		where $(a)$ and $(b)$ follow from the definition of ${\cU_{\fW^\star}(\gamma_0)}$. This completes the proof. \hspace*{\fill}$\Box$\par
	\endproof

 \subsection{Proof of \Cref{thm_completion_asym}}
	As will be shown next, controlling the right-hand side of \Cref{eq_ub_cU} is much easier for matrix completion.
	\begin{proposition}\label{prop_cU_completion}
		The following statements hold for the measurement matrices satisfying the matrix completion model (\Cref{asp_completion}).
  \begin{itemize}
      \item Suppose that $k>r$. There exists a true solution $\fW^\star = (\trueW_1,\trueW_2)$ such that
      \begin{equation*}
          \begin{aligned}
              &\max_{\Delta \fW\in {\cU_{\fW^\star}(t_0)}}\left\{ \frac{1}{m}\sum_{i\in{\cS_t}}|\langle A_i, \Delta W_1 \trueW_2\rangle|\right\}\geq \frac{1}{3}\cdot\sqrt{\frac{k-r}{d_2}}\cdot\sqrt{\frac{pp_0}{sd_1d_2}},
          \end{aligned}
      \end{equation*}
  with probability at least $1-\exp(-\Omega(spp_0d_1(k-r)))$.
  \item Suppose that $k\geq r$ and $\trueX$ is coherent. For any $\Gamma<\infty$, we have
  \begin{equation*}
      \begin{aligned}
          &\min_{\fW^\star\in \cW\cap \cB_{\mathrm{op}}(\Gamma)}\max_{\Delta \fW\in {\cU_{\fW^\star}(t_0)}}\left\{ \frac{1}{m}\sum_{i\in{\cS_t}}|\langle A_i, \Delta W_1 \trueW_2\rangle|\right\}\geq \frac{1}{3}\sqrt{\frac{1}{d_2}}\sqrt{\frac{pp_0}{sd_1d_2}}\cdot\frac{\sigma_r(\trueX)}{\Gamma},
      \end{aligned}
  \end{equation*}
  with probability at least $1-\exp(-\Omega(spp_0d_1))$.
  \end{itemize}
	\end{proposition}
\proof{Proof}
    To prove the first statement, we explicitly design a true solution and a perturbation that achieves the desired inequality. Consider the singular value decomposition of $\trueX = U^\star\Sigma^\star {V^{\star\top}}$.
	Let the solution $\fW^\star = (\trueW_1, \trueW_2)$ be defined as:
	\begin{equation}\label{eq_trueW}
		\trueW_1 = \begin{bmatrix}
			U^\star\Sigma^\star & 0_{d_1\times (k-r)}
		\end{bmatrix}, \ \ \trueW_2 = \begin{bmatrix}
			V^\star & I_{d_2\times (k-r)}
		\end{bmatrix}^\top.
	\end{equation}
 Clearly, we have $\trueW_1\trueW_2 = \trueX$, and hence, $\fW^\star \in \cW$. Let ${\bar{\Omega}} = \{({\alpha,\beta})\in {\Omega}: {\beta}\leq k-r, |E[{\alpha,\beta}]|\geq t_0\}$ and ${\Delta{\fW}}' = (\Delta W_1', 0_{k\times d_2})$ with
 \begin{equation*}
		\Delta W_1' = \begin{bmatrix}
					0_{d_1\times r} & \Delta W_{12}'
				\end{bmatrix}, \Delta W_{12}'\in \bR^{d_1\times (k-r)},
\end{equation*}
where 
\begin{equation}\label{eq_sol}
    \Delta W_{12}'[{\alpha,\beta}]=\begin{cases} \frac{\sign(E[{\alpha,\beta}])}{\sqrt{|{\bar{\Omega}}|}} & \text{if}\ ({\alpha,\beta})\in {\bar{\Omega}}\\
					0 & \text{otherwise}
				\end{cases}.
\end{equation}
		With this definition, we have $\Delta W_1'\trueW_2 = \begin{bmatrix}
			\Delta W_{12}' & 0_{d_1\times (d_2-k+r)}
		\end{bmatrix}$.
		Moreover, simple calculation reveals that $\Delta{\fW}'\in{\cU_{\fW^\star}(t_0)}$. 
		Therefore, for $\fW^\star = (\trueW_1,\trueW_2)$ defined as \Cref{eq_trueW}, we have
		\begin{equation}
		    \label{eq_cU_lb}
			\begin{aligned}
			    &\max_{\Delta \fW\in {\cU_{\fW^\star}(t_0)}}\left\{ \frac{1}{m}\sum_{i\in{\cS_t}}|\langle A_i, \Delta W_1 \trueW_2\rangle|\right\}\geq \frac{1}{m}\sum_{i\in{\cS_t}}|\langle A_i, \Delta W_1' \trueW_2\rangle|
                = \frac{1}{m}\sum_{({\alpha,\beta})\in{\bar{\Omega}}}|\Delta W_{12}'({\alpha,\beta})| = \frac{\sqrt{|{\bar{\Omega}}|}}{m}.
			\end{aligned}
		\end{equation}
		To finish the proof, we provide lower and upper bounds for $|{\bar{\Omega}}|$ and $m$, respectively. For any $1\leq {\alpha}\leq d_1, 1\leq {\beta}\leq d_2$, we have $({\alpha,\beta})\in {\Omega}$ with probability $s$. Therefore, $m = |{\Omega}|$ has a binomial distribution with parameters $(d_1d_2,s)$ and, according to \Cref{lem::concentration-Bernoulli}, we have $m\leq (3/2)d_1d_2s$ with probability at least $1-\exp(-\Omega(d_1d_2s))$. On the other hand, for any $1\leq {\alpha}\leq d_1, 1\leq {\beta}\leq k-r$, we have $({\alpha,\beta})\in {\bar{\Omega}}$ with probability $spp_0$, which in turn implies that $|{\bar{\Omega}}|$ also has a binomial distribution with parameters $(d_1(k-r), spp_0)$. Again,  \Cref{lem::concentration-Bernoulli} can be invoked to show that $|{\bar{\Omega}}|\geq spp_0d_1(k-r)/4$ with probability at least $1-\exp(-\Omega(spp_0d_1(k-r)))$. Therefore, a simple union bound implies that
		\begin{equation*}
		    \begin{aligned}
		        \frac{\sqrt{|{\bar{\Omega}}|}}{m}\geq\frac{1}{3}\sqrt{\frac{k-r}{d_2}}\sqrt{\frac{pp_0}{sd_1d_2}}\implies \max_{\Delta \fW\in {\cU_{\fW^\star}(t_0)}}\left\{ \frac{1}{m}\sum_{i\in{\cS_t}}|\langle A_i, \Delta W_1 \trueW_2\rangle|\right\}\geq\frac{1}{3}\sqrt{\frac{k-r}{d_2}}\sqrt{\frac{pp_0}{sd_1d_2}},
		    \end{aligned}
		\end{equation*}
		with probability at least $1-\exp(-\Omega(spp_0d_1(k-r)))$. This completes the proof of the first statement. 

  To prove the second statement, recall that since $\trueX$ is coherent, at least one of the columns of $U^\star$ or $V^\star$ is aligned with a standard unit vector. Without loss of generality, let us assume that the last column of $V^\star$ is aligned with the first standard unit vector $e_{1}$. Our subsequent argument can be readily extended to more general cases where an arbitrary column of $U^\star$ or $V^\star$ is aligned with an arbitrary standard unit vector. Based on our assumption, we have $V^\star = \begin{bmatrix}
      V' & e_{1}
  \end{bmatrix}$, for some orthonormal matrix $V'\in \bR^{d_2\times (r-1)}$. For any $(\trueW_1,\trueW_2)\in \cW$ with $\max\{\norm{\trueW_1},\norm{\trueW_2}\}\leq \Gamma$, define ${\Delta{\fW}}' = (\Delta W_1', 0_{k\times d_2})$, where $\Delta W_1' = Y{\Sigma^\star}^{-1}{U^\star}^\top W_1^\star$. To define $Y\in \bR^{d_1\times r}$, let ${{\Omega}}' = \{({\alpha}, 1)\in {\Omega}: |E[{\alpha}, 1]|\geq t_0\}$ and
  $$
  Y = \begin{bmatrix}
      0_{d_1\times (r-1)} & y
  \end{bmatrix}, \text{ where } y\in \bR^{d_1}, \text{and } y[{\alpha}] =  \begin{cases}
      \frac{\sigma_r(\trueX)}{\Gamma}\frac{\sign(E[{\alpha}, 1])}{\sqrt{|{\Omega}'|}} & \text{if $({\alpha}, 1)\in {\Omega}'$}\\
      0 & \text{otherwise}
  \end{cases}.
  $$
  Simple calculation reveals that $\Delta\fW'\in {\cU_{\fW^\star}(t_0)}$. Moreover, similar to the proof of the first statement, one can verify that 
\begin{equation*}
    \begin{aligned}
        \min_{\fW^\star\in \cW\cap \cB(\Gamma)}\max_{\Delta \fW\in {\cU_{\fW^\star}(t_0)}}\left\{ \frac{1}{m}\sum_{i\in{\cS_t}}|\langle A_i, \Delta W_1 \trueW_2\rangle|\right\} &\geq \frac{1}{m}\sum_{({\alpha}, 1)\in{{\Omega}}'}|\Delta y[{\alpha}]| \\
   &= \frac{\sqrt{|{{\Omega}}'|}}{m}\cdot \frac{\sigma_r(\trueX)}{\Gamma}\\
   &\geq \frac{1}{3}\sqrt{\frac{pp_0}{sd_1d_2^2}}\cdot\frac{\sigma_r(\trueX)}{\Gamma}
    \end{aligned}
\end{equation*}
with probability at least $1-\exp(-\Omega(spp_0d_1))$. This completes the proof of the second statement. \hspace*{\fill}$\Box$\par
\endproof

Now, we are ready to provide the proof of \Cref{thm_completion_asym}.\vspace{1mm}
	
	\noindent{\it Proof of \Cref{thm_completion_asym}.} We start with the proof of the third statement. {Combining \Cref{prop_BFU} and the second statement of \Cref{prop_cU_completion}, we have that, for every $\fW^\star\in \cW\cap \cB_{\mathrm{op}}(\Gamma)$, there exists a direction $\Delta\fW\in \cU_{\fW^\star}(t_0)$ such that}
	{\begin{equation*}
	    \begin{aligned}
	        \fc_{\ell_1}(\fW^\star+{\gamma}\Delta \fW) - \fc_{\ell_1}(\fW^\star)
  &\leq -{\gamma\cdot}\bigg(\frac{1}{m}\sum_{i\in{\cS_t}}|\langle A_i, \Delta W_1 \trueW_2\rangle|\bigg)\leq -\frac{1}{3} \sqrt{\frac{pp_0}{sd_1d_2^2}}\cdot\frac{\sigma_r(\trueX)\gamma}{\Gamma},
	    \end{aligned}
	\end{equation*}
	for any $\gamma\leq t_0$, with probability at least $1-\exp(\Omega(spp_0d_1r))$. 
 In light of the definition of the directional derivative, this implies that $f^{c'}_{\ell_1}(\fW^\star, \Delta\fW)<0$. 
 Therefore, $\fW^\star$ cannot be a critical point according to \Cref{lem_critical_mr}.} This completes the proof of the third statement. The proof of the first statement is identical and omitted for brevity. 
 
 Finally, we provide the proof of the second statement. For any $\fW^\star=(\trueW_1, \trueW_2)\in \cW(k-1)$, let $\rank(\trueW_1) = k_1$ and $\rank(\trueW_2) = k_2$. Indeed, we have $\dim\ker(\trueW_1) = k-k_1$ and $\dim\ker({\trueW_2}^\top) = k-k_2$, which in turn implies that $\dim\ker(\trueW_1)\cap \ker({\trueW_2}^\top)\geq k-(k_1+k_2)>0$. Let $S\in \bR^{k\times (k-k_1-k_2)}$ be a matrix whose columns form a basis for $\ker(\trueW_1)\cap \ker({\trueW_2}^\top)$. Let $({\bar\alpha}, {\bar\beta})$ be such that $|E[{\bar\alpha}, {\bar\beta}]|\geq t_0$. Indeed, such a pair exists with probability at least $1-\exp(-\Omega(spp_0d_1d_2))$. Let $\Delta W_1' = Y_1S^\top$ and $\Delta W_2' = SY_2$, for $Y_1\in \bR^{d_1\times (k-k_1-k_2)}$ and $Y_2\in \bR^{(k-k_1-k_2)\times d_2}$, where 
 $$
Y_1[{\alpha,\beta}] =  \begin{cases}
      {1} & \text{if $({\alpha,\beta})=({\bar\alpha}, 1)$}\\
      0 & \text{otherwise}
  \end{cases},\qquad
  Y_2[{\alpha,\beta}] =  \begin{cases}
      \sign(E[{\bar\alpha}, {\bar\beta}]) & \text{if $({\alpha,\beta})=(1, {\bar\beta})$}\\
      0 & \text{otherwise}
  \end{cases}.
 $$
 Similar to \Cref{eq_x1}, one can verify that
 \begin{equation*}
     \begin{aligned}
         &\fc_{\ell_1}(\fW^\star+{\gamma}\Delta\fW')-\fc_{\ell_1}(\fW^\star) 
        =\frac{1}{m}\sum_{i=1}^m|\langle A_i, {\gamma^2}Y_1Y_2\rangle-\epsilon_i|-|\epsilon_i|
		\leq -\frac{\gamma^2}{2sd_1d_2}
     \end{aligned}
 \end{equation*}
 with probability at least $1-\exp(-\Omega(spp_0d_1d_2))$. This completes the proof.
 $\hfill\square$

\subsection{Proof of \Cref{thm_no_benign_asym}}
 Next, we extend our analysis to the asymmetric matrix sensing. Our next two propositions characterize the local landscape of the asymmetric matrix sensing at the true solutions from the sets $\cW(2k/3)$ and $\cW\backslash\cW(2k/3)$. 

\begin{proposition}[Local landscape around $\cW\backslash\cW(2k/3)$]\label{prop_imbalanced}
                Consider~\ref{BM-sensing-asym} with measurement matrices satisfying  \Cref{asp_sensing}. Suppose that $m\lesssim d_1k$ and there exist constants $0<c_1\leq c_2<1$ such that $c_1d_1\leq d_2\leq c_2d_1$. {Then, with probability at least $1-\exp(-\Omega(mpp_0))$, for every $\fW^\star\in\cW\backslash\cW(2k/3)$, there exists a descent direction $\Delta \fW$ satisfying $\norm{\Delta \fW}_F=1$ such that, for every $\gamma\lesssim \sqrt{\frac{mpp_0}{d_1k}}\cdot\frac{t_0}{\max\left\{\sigma_{\min,+}(W_1^\star), \sigma_{\min,+}(W_2^\star)\right\}}$, we have
\begin{equation*}
    \begin{aligned}
        f^{\mathrm{s}}_{\ell_1}(\fW^\star+{\gamma}\Delta \fW) \!-\! f^{\mathrm{s}}_{\ell_1}(\fW)
    \lesssim -\left(\sqrt{\frac{d_1kpp_0}{m}} \min\left\{\sigma_{\min,+}(W_1^\star), \sigma_{\min,+}(W_2^\star)\right\}\right)\cdot\gamma.
    \end{aligned}
\end{equation*}   }
\end{proposition}

\begin{proposition}[Landscape around $\cW(2k/3)$]\label{prop_balanced}
Consider~\ref{BM-sensing-asym} with measurement matrices satisfying  \Cref{asp_sensing}. 
{Then, with probability at least $1-\exp(-\Omega(mpp_0))$, for every $\fW^\star\in\cW(2k/3)$, there exists a descent direction $\Delta \fW$ satisfying $\norm{\Delta \fW}_F=1$ such that, for every $\gamma\lesssim \sqrt{\frac{mpp_0}{d_1k}}t_0$, we have
\begin{equation*}
    \begin{aligned}
        f^{\mathrm{s}}_{\ell_1}(\fW^\star+{\gamma}\Delta \fW) - f^{\mathrm{s}}_{\ell_1}(\fW)
&\lesssim -\sqrt{\frac{d_1pp_0}{m}}\gamma^2.
    \end{aligned}
\end{equation*}}   
\end{proposition}
Before presenting the proofs of Propositions~\ref{prop_imbalanced} and~\ref{prop_balanced}, we use them to complete the proof of \Cref{thm_no_benign_asym}. 

\noindent{\it Proof of \Cref{thm_no_benign_asym}.} 
The first statement follows directly from \Cref{prop_balanced}. The second statement is a direct consequence of \Cref{prop_imbalanced}.
$\hfill\square$

In the remainder of this section, we present the proofs of \Cref{prop_imbalanced,prop_balanced}. 

\subsection{Proof of \Cref{prop_imbalanced}}
\begin{sloppypar}
    {Recall that, for asymmetric matrix completion, we provided an upper bound on $f^{\mathrm{c}}_{\ell_1}(\fW^\star+{\gamma}\Delta \fW) - f^{\mathrm{c}}_{\ell_1}(\fW)$ by choosing an explicit direction in ${\cU_{\fW^\star}(\gamma_0)}$ that lead to the desired result.} However, when the measurement matrices do not follow the matrix completion model, this explicit direction may no longer belong to ${\cU_{\fW^\star}(\gamma_0)}$. Under such circumstances, it may be non-trivial to explicitly construct a solution in ${\cU_{\fW^\star}(\gamma_0)}$. 
\end{sloppypar}

{To address this issue, we consider a further refinement of ${\cU_{\fW^\star}(\gamma_0)}$. 
Let us define $\cW_L(k/3)= \{(W_1,W_2)\in \cW: \rank(W_1)\geq k/3\}$ and $\cW_R(k/3) = \{(W_1,W_2)\in \cW: \rank(W_2)\geq k/3\}$. Evidently, we have $\cW\backslash\cW(2k/3) \subset \cW_L(k/3)\cup \cW_R(k/3)$. Therefore, to prove \Cref{prop_imbalanced}, it suffices to show:}
{\begin{itemize}
    \item For every $\fW^\star\in\cW_R(k/3)$, there exists a direction $\Delta\fW$ satisfying $\norm{\Delta\fW}_F=1$ such that, for every $\gamma\lesssim \sqrt{\frac{mpp_0}{d_1k}}\cdot\frac{t_0}{\sigma_{\min,+}(W_2^\star)}$, we have
\begin{equation}
    \begin{aligned}
& f^{\mathrm{s}}_{\ell_1}(\fW^\star+{\gamma}\Delta \fW) - f^{\mathrm{s}}_{\ell_1}(\fW)\lesssim -\sqrt{\frac{pp_0d_1k}{m}}\sigma_{\min,+}(W_2^\star)\cdot \gamma,\quad \text{ if $m\lesssim d_1k$},
    \end{aligned}
    \label{ub_rhs}
\end{equation}
\item For every $\fW^\star\in\cW_L(k/3)$, there exists a direction $\Delta\fW$ satisfying $\norm{\Delta\fW}_F=1$ such that, for every $\gamma\lesssim \sqrt{\frac{mpp_0}{d_1k}}\cdot\frac{t_0}{\sigma_{\min,+}(W_1^\star)}$, we have
\begin{equation}
    \begin{aligned}
&f^{\mathrm{s}}_{\ell_1}(\fW^\star+{\gamma}\Delta \fW) - f^{\mathrm{s}}_{\ell_1}(\fW)\lesssim -\sqrt{\frac{pp_0d_2k}{m}}\sigma_{\min,+}(W_1^\star)\cdot\gamma,\quad \text{ if $m\lesssim d_2k$}.
    \end{aligned}
    \label{ub_lhs}
\end{equation}   
\end{itemize}}
In what follows, we prove the correctness of \Cref{ub_rhs}. The proof of \Cref{ub_lhs} follows identically in light of our assumption $c_1d_1\leq d_2\leq c_2d_1$. To this end, we consider the following refinement of ${\cU_{\fW^\star}(\gamma_0)}$ (a set previously defined in~\Cref{eq_cU}):
	{\begin{equation}
		\begin{aligned}\label{eq_cV}
			{\cV_{\fW^\star}(\gamma_0)} := \big\{ (U_1, 0_{k\times d_2}): &
			U_1\in \bR^{d_1\times k}, \norm{U_1}_F\leq {1},\\
			&\langle A_i, U_1 W^\star_2 \rangle = 0, &&\forall i\in \bar{\cS}\cup (\cS\backslash \cS_t),\\
			&{\gamma_0}\langle A_i, U_1 W^\star_2 \rangle = t_0\sign(\epsilon_i), &&\forall i\in {\cS_t}\big\}.
		\end{aligned}
	\end{equation}}
	It is easy to verify that ${\cV_{\fW^\star}(\gamma_0)}\subseteq {\cU_{\fW^\star}(\gamma_0)}\subseteq {\cB}$ for every ${\gamma_0>0}$. {Therefore, Proposition~\ref{prop_BFU} can be invoked, which implies that, for every $\fW^\star\in \cW_R(k/3)$ and $\Delta\fW\in \cV_{\fW^\star}(\gamma_0)$ (assuming that ${\cV_{\fW^\star}(\gamma_0)}$ is non-empty), we have
	\begin{equation}\label{eq_cv_ub}
		\begin{aligned}
			 f^{\mathrm{s}}_{\ell_1}(\fW^\star+{\gamma}\Delta \fW) - f^{\mathrm{s}}_{\ell_1}(\fW)
			&\leq -{\gamma\cdot}\Big(\frac{1}{m}\sum_{i\in{\cS_t}}|\langle A_i, \Delta W_1 \!\trueW_2\rangle|\Big)\leq -\frac{|\cS_t|}{m}\cdot{\frac{t_0}{\gamma_0}\cdot\gamma},\quad \text{for every $0\leq \gamma\leq \gamma_0$,}
		\end{aligned}
	\end{equation} }
where the last inequality follows from the definition of ${\cV_{\fW^\star}(\gamma_0)}$. Consequently, for an appropriate choice of {$\gamma_0$}, the above inequality achieves the desired result.
	However, unlike ${\cU_{\fW^\star}(\gamma_0)}$, the set ${\cV_{\fW^\star}(\gamma_0)}$ is not guaranteed to be non-empty in general. To see this, note that the norm of any solution for the system of linear equations in \Cref{eq_cV} increases with $t_0/\gamma_0$; as a result, the condition $\norm{U_1}_F\leq {1}$ would be violated for sufficiently large $t_0/\gamma_0$. The following lemma provides sufficient conditions under which ${\cV_{\fW^\star}(\gamma_0)}$ is non-empty.
	\begin{lemma}\label{lem_cV}
		The set ${\cV_{\fW^\star}(\gamma_0)}$ is non-empty if the following conditions are satisfied:
		\begin{itemize}
			\item The matrix $A_{\fW^\star} = \begin{bmatrix}
				\vecc\left(A_1 {\trueW_2}^\top\right) & \dots & \vecc\left(A_m{\trueW_2}^\top\right)
			\end{bmatrix}^\top$ is full row-rank.
			\item We have ${\frac{t_0}{\gamma_0}\leq \left(\sqrt{|\cS_t|}\norm{A_{\fW^\star}^\top \left(A_{\fW^\star}A_{\fW^\star}^\top\right)^{-1}}\right)^{-1}}$.
		\end{itemize}
	\end{lemma}
	\proof{Proof}
		For every $1\leq i\leq m$, we have $\langle A_i, U_1\trueW_2\rangle=\langle A_i{\trueW_2}^\top, U_1\rangle$. Therefore, to show the non-emptiness of ${\cV_{\fW^\star}(\gamma_0)}$, it suffices to show that the following system of linear equations has a solution $U_1$ that satisfies $\norm{U_1}_F\leq {1}$:
		\begin{equation*}
			\begin{cases}
				\inner{A_i {\trueW_2}^\top}{U_1} = 0, &\forall i\in \bar{\cS}\cup (\cS\backslash \cS_t),\\
				\inner{A_i {\trueW_2}^\top}{U_1} = {\frac{t_0}{\gamma_0}}sign(\epsilon_i), &\forall i\in {\cS_t}.
			\end{cases}
		\end{equation*}
		The above system of linear equations can be written as
		\begin{equation}\label{eq_lin}
			A_{\fW^\star} u = b, \quad \text{where}\quad \text{$A_{\fW^\star} = \begin{bmatrix}
					\vecc\left(A_1 {\trueW_2}^\top\right)^\top\\
					\vdots\\
					\vecc\left(A_m{\trueW_2}^\top\right)^\top
				\end{bmatrix}$, } u = \vecc(U_1), b_i = \begin{cases}
					{\frac{t_0}{\gamma_0}}\sign(\epsilon_i) & \text{if } i\in \cS_t\\
					0 & \text{if } i\not\in \cS_t
				\end{cases}.
		\end{equation}
		Since $A_{\fW^\star}$ is assumed to be full row-rank, the above set of equations is guaranteed to have a solution. To show the existence of a solution $U_1$ that satisfies $\norm{U_1}_F\leq {1}$, we consider the least-norm solution of $A_{\fW^\star} u = b$, which is equal to $u^\star = A_{\fW^\star}^{\dagger}b$, where $A_{\fW^\star}^\dagger$ is the pseudo-inverse of $A_{\fW^\star}$ defined as $A_{\fW^\star}^{\dagger} = A_{\fW^\star}^\top \left(A_{\fW^\star}A_{\fW^\star}^\top\right)^{-1}$. This in turn implies that $\norm{u^\star}\leq \norm{A_{\fW^\star}^\top \left(A_{\fW^\star}A_{\fW^\star}^\top\right)^{-1}}\norm{b}\leq \sqrt{|\cS_t|}\norm{A_{\fW^\star}^\top \left(A_{\fW^\star}A_{\fW^\star}^\top\right)^{-1}}{\frac{t_0}{\gamma_0}}\leq {1}$,
		where the last inequality follows from the second condition. This completes the proof. \hspace*{\fill}$\Box$\par
	\endproof

Next, we apply the above lemma to asymmetric matrix sensing and show that, with high probability, ${\cV_{\fW^\star}(\gamma_0)}$ is non-empty for every $\fW^\star\in \cW_R(k/3)$.
\begin{lemma}\label{prop_cV_nonempty_v2}
                {Suppose that the measurement matrices follow the matrix sensing model in  \Cref{asp_sensing}. Suppose that $m\lesssim d_1k$. Then, with probability at least $1-\exp(-\Omega(mpp_0))$, the set ${\cV_{\fW^\star}(\gamma_0)}$ is non-empty for every $\fW^\star\in \cW_R(k/3)$ and $0<{\frac{t_0}{\gamma_0}\lesssim\sqrt{\frac{d_1k}{mpp_0}}\sigma_{\min,+}(W_2^\star)}$.}
\end{lemma}
Before presenting the proof of the above lemma, we show how it can be used to finish the proof of \Cref{prop_imbalanced}.

{\noindent{\it Proof of \Cref{prop_imbalanced}.} As mentioned before, to prove this lemma, it suffices to show the validity of \Cref{ub_rhs} and \Cref{ub_lhs}. For brevity, we will only prove the former, as the latter follows in an identical manner. Let us choose $\gamma_0 \asymp \sqrt{\frac{mpp_0}{d_1k}}\cdot\frac{t_0}{\sigma_{\min,+}(W_2^\star)}$. This choice of $\gamma_0$ satisfies the condition of \Cref{prop_cV_nonempty_v2}. Therefore, ${\cV_{\fW^\star}(\gamma_0)}$ is non-empty for every $\fW^\star\in \cW_R(k/3)$, with probability at least $1-\exp(-\Omega(mpp_0))$. Conditioned on the non-emptiness of ${\cV_{\fW^\star}(\gamma_0)}$, \Cref{eq_cv_ub} immediately implies \Cref{ub_rhs}. $\hfill\square$}

Next, we continue with the proof of \Cref{prop_cV_nonempty_v2}. To this goal, it suffices to prove that there exists $\xi>0$ such that $\sigma_{\min}(A_{\fW^\star})\geq \xi$ for all ${\fW^\star\in \cW_R(k/3)}$. To see this, note that $\sigma_{\min}(A_{\fW^\star})\geq \xi$ implies that $A_{\fW^\star}$ is full row-rank. Moreover, 
\begin{equation*}
    \begin{aligned}
        \norm{A_{\fW^\star}^\top \left(A_{\fW^\star}A_{\fW^\star}^\top\right)^{-1}}
    &\leq \norm{A_{\fW^\star}^\top \left(A_{\fW^\star}A_{\fW^\star}^\top\right)^{-1/2}}\norm{ \left(A_{\fW^\star}A_{\fW^\star}^\top\right)^{-1/2}}\\
                        &\leq \norm{ \left(A_{\fW^\star}A_{\fW^\star}^\top\right)^{-1/2}}=\frac{1}{\sigma_{\min}(A_{\fW^\star})}\leq \frac{1}{\xi}.
    \end{aligned}
\end{equation*}
This implies that \Cref{lem_cV} is satisfied for ${\frac{t_0}{\gamma_0}}\leq {\xi}/\sqrt{|\cS_t|}$. Our next lemma provides such uniform lower bound for $\sigma_{\min}(A_{\fW^\star})$.
{\begin{lemma}\label{lem_min_SVD}
    Suppose that $m\lesssim d_1k$ and $d_1\geq cd_2$, for some constant $c>1$.
    Then, with probability at least $1-\exp(-\Omega(d_1k))$, for every $\fW^\star\in \cW_R(k/3)$, we have $$\sigma_{\min}(A_{\fW^\star})\gtrsim \sqrt{d_1k}\sigma_{\min,+}(W_2^\star).$$
\end{lemma}
\proof{Proof}
    Consider the following characterization of $A_{\fW^\star}$:
    \begin{equation*}
                    A_{\fW^\star}=\underbrace{\begin{bmatrix}
                        \vecc\left(A_1 \right)^\top\\
                        \vdots\\
                        \vecc\left(A_m\right)^\top
                    \end{bmatrix}}_{:= A\in \bR^{m\times d_1d_2}}\cdot \underbrace{{\trueW_2}^\top\otimes I_{d_1\times d_1}}_{:=\tilde W_2\in \bR^{d_1d_2\times d_1k}}.
                \end{equation*}
                Let $\rank(\trueW_2) = k'\geq k/3$.
                Consider the SVD of $\trueW_2$ as $\trueW_2=U\Sigma V^{\top}$ where $U\in \bR^{k\times k'}$, $\Sigma\in \bR^{k'\times k'}$ and $V\in \bR^{d_2\times k'}$. Then, one can write 
                \begin{equation*}
                    \begin{aligned}
                        &A \tilde W_2
                     =\underbrace{A\cdot (V\otimes I_{d_1\times d_1})}_{:= A_V}\cdot {(\Sigma U^\top\otimes I_{d_1\times d_1})}\implies\sigma_{\min}(A \tilde W_2)\geq \sigma_{\min}(A_V)\sigma_{\min}(\Sigma U^\top\otimes I_{d_1\times d_1}).
                    \end{aligned}
                \end{equation*}
                Note that $\sigma_{\min}(\Sigma U^\top\otimes I_{d_1\times d_1})\geq \sigma_{\min,+}(W_2^\star)$ for every $\fW^\star\in \cW_R(k/3)$. Moreover, according to \Cref{lem_uniform_svd}, with probability at least $1-\exp(-\Omega(d_1k'))$, we have $\sigma_{\min}(A_V)\gtrsim \sqrt{d_2k'}$ for every $V\in \cO_{d_2\times k'}$.
                Therefore, $\sigma_{\min}(A_{\fW^\star})\geq \sigma_{\min}(A \tilde W_2)\gtrsim \sqrt{d_2k'}\sigma_{\min,+}(W_2^\star)$ with probability at least $1-\exp(-\Omega(d_1k'))$. The proof is completed by noting that $k'\geq k/3$. \hspace*{\fill}$\Box$\par
\endproof}

Equipped with this lemma, we are ready to present the proof of \Cref{prop_cV_nonempty_v2}.

\noindent{\it Proof of \Cref{prop_cV_nonempty_v2}.} To prove this lemma, it suffices to show that the conditions of \Cref{lem_cV} are satisfied for every $\fW^\star\in \cW_R(k/3)$ and $0<{\frac{t_0}{\gamma_0}\lesssim\sqrt{\frac{d_1k}{mpp_0}}\sigma_{\min,+}(W_2^\star)}$. First, note that, due to \Cref{lem_min_SVD}, with probability at least $1-\exp(-\Omega(d_1k))$, $A_{\fW^\star}$ is full row-rank for every $\fW^\star\in \cW_R(k/3)$. Moreover, for every $\fW^\star\in \cW_R(k/3)$,
{\begin{equation*}
    \begin{aligned}
        &\left(\sqrt{|\cS_t|}\norm{A_{\fW^\star}^\top \left(A_{\fW^\star}A_{\fW^\star}^\top\right)^{-1}}\right)^{-1}\geq ({|\cS_t|})^{-1/2}\sigma_{\min}(A_{\fW^\star})\gtrsim \sqrt{\frac{d_1k}{mpp_0}}\sigma_{\min,+}(W_2^\star)\gtrsim\frac{t_0}{\gamma_0}
    \end{aligned}
\end{equation*}}
where the second inequality follows from \Cref{lem_min_SVD} and the fact that $|\cS_t|\leq 2mpp_0$ with probability at least $1-\exp(-\Omega(mpp_0))$. Therefore, the second condition of \Cref{lem_cV} is also satisfied for every $\fW^\star\in \cW_R(k/3)$ with probability at least $1-\exp(-\Omega(mpp_0))-\exp(-\Omega(d_1k))$. This completes the proof.$\hfill\square$

\subsection{Proof of \Cref{prop_balanced}}
\begin{sloppypar}
    For any $\fW^\star = (\trueW_1,\trueW_2)\in \cW(2k/3)$, let $\rank(\trueW_1) =k_1$ and $\rank(\trueW_2)=k_2$. Indeed, we have $\dim\ker(\trueW_1)\geq k-k_1$ and $\dim\ker\left({\trueW_2}^\top\right)\geq k-k_2$, which leads to $\dim\ker\left(\trueW_1\cap {\trueW_2}^\top\right) = k-(k_1+k_2):=k'\geq k/3$. Suppose that $S_{\fW^\star}\in \bR^{k\times k'}$ is {an orthonormal} matrix whose columns form a basis for $\ker\left(\trueW_1\cap {\trueW_2}^\top\right)$. Define the set of perturbations 
\end{sloppypar}
$$
{\cZ_{\fW^\star}} = \left\{(Y_1S_{\fW^\star}^\top, S_{\fW^\star}Y_2): Y_1\in \bR^{d_1\times k'}, Y_2\in \bR^{k'\times d_2}, \norm{Y_1}_F\leq {\frac{1}{\sqrt{2}}}, \norm{Y_2}_F\leq {\frac{1}{\sqrt{2}}}\right\}.
$$
Then, we have ${\cZ_{\fW^\star}}\in {\cB}$.
Moreover, for every $\Delta\fW = (\Delta W_1, \Delta W_2)\in {\cZ_{\fW^\star}}$, we have
\begin{equation*}
            \begin{aligned}
                &(\trueW_1+{\gamma}\Delta W_1)(\trueW_2+{\gamma}\Delta W_2)-\trueW_1\trueW_2={\gamma^2\cdot}Y_1Y_2.
            \end{aligned}
\end{equation*}
{Therefore, for every $\fW^\star\in\cW(2k/3)$ and $\Delta\fW\in \cZ_{\fW^\star}$, we have}
        {\begin{equation}\label{eq_ub_balanced}
            \begin{aligned}
                &f^{\mathrm{s}}_{\ell_1}(\fW^\star+{\gamma}\Delta\fW)-f^{\mathrm{s}}_{\ell_1}(\fW^\star)= {\frac{1}{m}\sum_{i=1}^{m}\left(|\langle A_i, {\gamma^2}Y_1Y_2\rangle-\epsilon_i|-|\epsilon_i|\right)}.
            \end{aligned}
        \end{equation}}
        A key property of the above equation is the fact that $Y_1Y_2$ is independent of the choice of the true solution $\fW^\star$. In other words, for every choice of $\norm{Y_1}_F\leq {\frac{1}{\sqrt{2}}}$ and $\norm{Y_1}_F\leq {\frac{1}{\sqrt{2}}}$, there exists $\Delta \fW\in \cZ_{\fW^\star}$ that satisfies the above equation. Therefore, it suffices to show the existence of $\norm{Y_1}_F\leq {\frac{1}{\sqrt{2}}}$ and $\norm{Y_1}_F\leq {\frac{1}{\sqrt{2}}}$ that achieves the desired bound in \Cref{prop_balanced}.
      
        To achieve this goal,  our approach will be similar to the proof of \Cref{prop_imbalanced}. Let us fix $Y_2 = {\frac{1}{2\sqrt{k'}}}\cdot V$ for some arbitrary orthonormal matrix $V\in \bR^{k'\times d_2}$. Clearly, we have ${\norm{Y_2}_F\leq \frac{1}{\sqrt{2}}}$. {Inspired by the definition of $\cV_{\fW^\star}(\gamma_0)$ in \Cref{eq_cV},} consider the following perturbation set for $Y_1$:
        \begin{equation}
		\begin{aligned}\label{eq_cV2}
			{\cV_{V}(\gamma_0)} := \big\{ U:
			&\ \norm{U}_F\leq {\frac{1}{\sqrt{2}}},\\
			&\langle A_i, U V \rangle = 0, \quad\forall i\in \bar{\cS}\cup (\cS\backslash \cS_t),\\
			& {\sqrt{2}\gamma_0}\langle A_i, U V \rangle = {t_0}\sign(\epsilon_i), \quad\forall i\in {\cS_t}\big\}.
		\end{aligned}
	\end{equation}
 Our next lemma shows that, for an appropriate choice of ${\frac{t_0}{\gamma_0}}$, the set ${\cV_V(\zeta)}$ is non-empty with high probability.
\begin{lemma}\label{lem_cV2}
    Suppose that the measurement matrices follow the matrix sensing model in  \Cref{asp_sensing}. Suppose that $m\leq d_1k$ and ${\frac{t_0}{\gamma_0}\lesssim\sqrt{\frac{d_1k}{mpp_0}}}$. Then, with probability at least $1-\exp(-\Omega(mpp_0))$, the set ${\cV_V(\gamma_0)}$ is non-empty.
\end{lemma}
The proof of \Cref{lem_cV2} follows from that of \Cref{prop_cV_nonempty_v2}, and hence, omitted for brevity. 

{Let us choose ${\gamma_0\asymp \sqrt{\frac{mpp_0}{d_1k}}t_0}$. This choice of $\gamma_0$ satisfies the condition of \Cref{lem_cV2}. Therefore, the set ${\cV_V(\gamma_0)}$ is non-empty with probability at least $1-\exp(-\Omega(mpp_0))$. Conditioned on this event and recalling that $Y_2 = V$, there exists $Y_1\in \cV_{V}(\gamma_0)$ such that }
 {\begin{equation}\label{eq_A}
     \begin{aligned}
         \frac{1}{m}\sum_{i=1}^{m}\left(|\langle A_i, {\gamma^2}Y_1Y_2\rangle-\epsilon_i|-|\epsilon_i|\right) &= -\frac{{\gamma^2}}{\sqrt{k'}}\cdot\frac{1}{m}\sum_{i\in \cS_t}|\inner{A_i}{Y_1V}| \\
         &= -\frac{{\gamma^2}}{2\sqrt{k'}}\cdot\frac{|\cS_t|}{m}\cdot\frac{t_0}{\gamma_0}\\
         &\lesssim -\frac{{\gamma^2}}{\sqrt{k'}}\frac{|\cS_t|}{m}\sqrt{\frac{d_1k}{mpp_0}}\\
     &\lesssim -\sqrt{\frac{d_1pp_0}{m}}\gamma^2, \qquad \text{for every $0\leq \gamma\leq \gamma_0$}.
     \end{aligned}
 \end{equation} }
In the last inequality, we used the facts that $k'\leq k$, and $|\cS_t|\leq \frac{3mpp_0}{2}$ with probability at least $1-\exp(-\Omega(mpp_0))$. This completes the proof of Proposition~\ref{prop_balanced}.$\hfill\square$
	
	\section{Matching Lower Bounds}\label{sec_lb}
	In this section, we present the proof of \Cref{thm_benign_l1}. As will be explained later, the proof of \Cref{thm_critical_sensing_sym} follows analogously. We start with the proof of the second statement of \Cref{thm_benign_l1}. In particular, we show that if $m\gtrsim \max\{d_1,d_2\}k/(1-2p)^4$, the true solutions coincide with the global optima of~\ref{BM-sensing-asym}. Our proof is the generalization of~\cite[Theorem 2.3]{ding2021rank} to the asymmetric setting.
 
	It suffices to show that for any $\fW = (W_1,W_2)\not\in \cW$ and $\fW^\star = (\trueW_1, \trueW_2)\in\cW$, we have $f^{\mathrm{s}}_{\ell_1}(\fW)-f^{\mathrm{s}}_{\ell_1}(\fW^\star)>0$. To this end, let $\Delta X = W_1W_2-\trueW_1\trueW_2$. Then, one can write
	\begin{equation*}
	    \begin{aligned}
	        f^{\mathrm{s}}_{\ell_1}(\fW)-f^{\mathrm{s}}_{\ell_1}(\fW^\star)
		&\geq \frac{1}{m}\sum_{i\in\bar\cS}\left|\inner{A_i}{\Delta X}\right|+\frac{1}{m}\sum_{i\in\cS}\left|\inner{A_i}{\Delta X}-\epsilon_i\right|-|\epsilon_i|\\
		&\geq \underbrace{\frac{1}{m}\sum_{i\in\bar\cS}\left|\inner{A_i}{\Delta X}\right|-\frac{1}{m}\sum_{i\in\cS}\left|\inner{A_i}{\Delta X}\right|}_{F}
	    \end{aligned}
	\end{equation*}
	where the second inequality follows from the triangle inequality. We will show that if the corruption probability $p$ is strictly less than half, then $F>0$ with high probability. To this goal, we use an important property of Gaussian measurements, called {\it $\ell_1/\ell_2$-restricted isometry property} ($\ell_1/\ell_2$-RIP), which is provided in the following lemma.
	\begin{lemma}[$\ell_1/\ell_2$-RIP~\cite{li2020nonconvex, charisopoulos2021low}]\label{lem_RIP}
		Suppose that the measurements follow the matrix sensing model in  \Cref{asp_sensing}.
		For all rank-$2r'$ matrices $X\in \bR^{d_1\times d_2}$ and any $\delta>0$, we have with probability at least $1-\exp(-\Omega(m\delta^2))$:
		$$
		\left(\sqrt{\frac{2}{\pi}}-\delta\right)\norm{X}_F\leq \frac{1}{m}\sum_{i=1}^m\left|\inner{A_i}{X}\right|\leq \left(\sqrt{\frac{2}{\pi}}+\delta\right)\norm{X}_F,
		$$
		provided that $m\geq c\max\{d_1,d_2\}r'/\delta^2$ for some constant $c>0$.
	\end{lemma}
    {\begin{remark}
        We note that the original version of the above lemma, as presented in \cite{li2020nonconvex, charisopoulos2021low}, only covers the case $0 < \delta < \sqrt{2/\pi}$. However, the proof in \cite{li2020nonconvex} remains valid for $\delta \geq \sqrt{2/\pi}$. In this setting, while the provided lower bound becomes vacuous, the upper bound remains valid and will be used in our subsequent analysis.
    \end{remark}}
	Equipped with the above lemma, we are ready to present the proof of the second statement of \Cref{thm_benign_l1}.
	\vspace{1mm}
	
	\noindent{\it Proof of the second statement of \Cref{thm_benign_l1}.}
	Our goal is to show that, if the measurement matrices follow the matrix sensing model, then with an overwhelming probability, we have $F\gtrsim \norm{\Delta X}_F$. This in turn implies that if $\Delta X\not=0$ (or equivalently $\fW\not\in\cW$), then $f^{\mathrm{s}}_{\ell_1}(\fW)-f^{\mathrm{s}}_{\ell_1}(\fW^\star)>0$.
	Notice that $\rank(\Delta X)\leq k+r$.
	Let $\cS_{k+r}^{d_1\times d_2} = \{X: X\in \bR^{d_1\times d_2}, \rank(X)\leq k+r, \norm{X}_F = 1\}$. It is easy to see that a sufficient condition for $F\gtrsim \norm{\Delta X}_F$ is to have
	\begin{equation*}
	    \min_{\Delta X\in \cS_{k+r}^{d_1\times d_2}}\!\!\Big\{\!\frac{1}{m}\!\sum_{i\in\bar\cS}\!\left|\!\inner{A_i}{\Delta X}\!\right|\!-\!\frac{1}{m}\!\sum_{i\in\cS}\!\left|\!\inner{A_i}{\Delta X}\!\right|\!\Big\}\!\gtrsim\! 1\!.
	\end{equation*}
	To prove the above inequality, we write
	\begin{equation*}
	    \begin{aligned}
	        \min_{\Delta X\in \cS_{k+r}^{d_1\times d_2}}\left\{\frac{1}{m}\sum_{i\in\bar\cS}\left|\inner{A_i}{\Delta X}\right|\!-\!\frac{1}{m}\sum_{i\in\cS}\left|\inner{A_i}{\Delta X}\right|\right\}\geq &\underbrace{\min_{\Delta X\in \cS_{k+r}^{d_1\times d_2}}\left\{\frac{1}{m}\sum_{i\in\bar\cS}\left|\inner{A_i}{\Delta X}\right|\right\}}_{F_1}\\
        &-\underbrace{\max_{\Delta X\in \cS_{k+r}^{d_1\times d_2}}\left\{\frac{1}{m}\sum_{i\in\cS}\left|\inner{A_i}{\Delta X}\right|\right\}}_{F_2}.
	    \end{aligned}
	\end{equation*}  
	We next show that $F_1-F_2> 0$. To this goal, we provide separate lower and upper bounds for $F_1$ and $F_2$. Conditioned on $|\bar\cS|$, $\ell_1/\ell_2$-RIP (\Cref{lem_RIP}) with $\delta = \sqrt{\frac{c\max\{d_1,d_2\}(k+r)}{|\bar\cS|}}$ can be invoked to show that
	\begin{equation*}
	    \begin{aligned}
	        F_1 &= \frac{|\bar\cS|}{m}\min_{\Delta X\in \cS_{k+r}^{d_1\times d_2}}\left\{\frac{1}{|\bar\cS|}\sum_{i\in\bar\cS}\left|\inner{A_i}{\Delta X}\right|\right\}\geq \frac{|\bar\cS|}{m}\left(\sqrt{\frac{2}{\pi}}-\sqrt{\frac{c\max\{d_1,d_2\}(k+r)}{|\bar\cS|}}\right)
	    \end{aligned}
	\end{equation*}
	with probability at least $1-\exp(-\Omega(\max\{d_1,d_2\}(k+r)))$. With the same probability, we have 
	\begin{equation*}
	    \begin{aligned}
	        F_2 &= \frac{|\cS|}{m}\max_{\Delta X\in \cS_{k+r}^{d_1\times d_2}}\left\{\frac{1}{|\cS|}\sum_{i\in\cS}\left|\inner{A_i}{\Delta X}\right|\right\}\leq \frac{|\cS|}{m}\left(\sqrt{\frac{2}{\pi}}+\sqrt{\frac{c\max\{d_1,d_2\}(k+r)}{|\cS|}}\right).
	    \end{aligned}
	\end{equation*}
	Combining these two inequalities, we have
	\begin{align}
  \label{eq_lb_E23}
		    F_1 - F_2\nonumber 
        &\geq \sqrt{\frac{2}{\pi}}\left(\frac{|\bar\cS|-|\cS|}{m}\right)-\left(\sqrt{\frac{|\bar\cS|}{m}}+\sqrt{\frac{|\cS|}{m}}\right)\sqrt{\frac{c\max\{d_1,d_2\}(k+r)}{m}}\\
		&\geq \sqrt{\frac{2}{\pi}}\left(\frac{m-2|\cS|}{m}\right)-2\sqrt{\frac{2c\max\{d_1,d_2\}k}{m}}
		\end{align}
	with probability at least $1-\exp(-\Omega(\max\{d_1,d_2\}(k+r)))$. On the other hand, since $|\cS|$ has a binomial distribution with parameters $(m,p)$ and $0\leq p<1/2$, \Cref{lem_bernoulli_lb} can be invoked to show that $\frac{m-2|\cS|}{m}\geq \min\left\{\frac{1}{4}, (1-2p)^2\right\}$ with probability at least $1-\exp(-\Omega((1-2p)^2m))$. Combining this inequality with \Cref{eq_lb_E23}, we have 
	\begin{equation*}
	    \begin{aligned}
	        F_1 - F_2
        &\geq\sqrt{\frac{2}{\pi}}\min\left\{\frac{1}{4},(1-2p)^2\right\}-2\sqrt{\frac{2c\max\{d_1,d_2\}k}{m}}\geq \left(\sqrt{\frac{2}{\pi}}-1\right)\min\left\{\frac{1}{4}, (1-2p)^2\right\}>0,
	    \end{aligned}
	\end{equation*}
	\begin{sloppypar}
		\noindent with probability at least $1-\exp(-\Omega(\max\{d_1,d_2\}(k+r)))-\exp(-\Omega((1-2p)^2m) = 1-\exp(-\Omega(\max\{d_1,d_2\}k))$ assuming that $m\geq \frac{512c\max\{d_1,d_2\}k}{(1-2p)^4}$. This completes the proof. $\hfill\square$\vspace{2mm}
	\end{sloppypar}
 Our next goal is to provide the proof of the first statement of \Cref{thm_benign_l1}. We first note that in this case, $m\gtrsim\frac{\max\{d_1,d_2\}r}{(1-2p)^4}$. Therefore, the sample size is not sufficiently large for $\ell_1/\ell_2$-RIP to control the deviation of the loss uniformly over the perturbation set $\cS_{k+r}^{d\times d}$. To circumvent this issue, we further decompose $\Delta X$ into two parts. Consider $\fW = \fW^\star+{\gamma}\Delta \fW$ for some $\norm{\Delta \fW}_F = \sqrt{\norm{\Delta W_1}_F^2+\norm{\Delta W_2}_F^2}\leq {1}$. Therefore, we have 
 \begin{equation*}
     \begin{aligned}
         \Delta X &= W_1W_2 - \trueW_1{\trueW_2} = {\gamma\cdot}\underbrace{\trueW_1\Delta W_2+ \Delta W_1{\trueW_2}}_{\Delta X_1}+{\gamma^2\cdot}\underbrace{\Delta W_1 \Delta W_2}_{\Delta X_2}.
     \end{aligned}
 \end{equation*}
 For every balanced true solution $\fW^\star\in \cW(2r)$, 
	we have $\rank(\Delta X_1)\leq 2r$. As a result, the effect of $\Delta X_1$ can be controlled via $\ell_1/\ell_2$-RIP with $m\gtrsim\frac{\max\{d_1,d_2\}r}{(1-2p)^4}$. On the other hand, $\rank(\Delta X_2)$ may be as large as $k$, but its magnitude is much smaller than $\Delta X_1$ {since it is controlled by the second-order term $\gamma$}. With this insight, we present the proof of the first statement of \Cref{thm_benign_l1}.

	\vspace{1mm}
	\noindent {\it Proof of first statement of \Cref{thm_benign_l1}.} 
	We have
	\begin{equation*}
	    \begin{aligned}
	        f^{\mathrm{s}}_{\ell_1}(\fW^\star{+\gamma\Delta\fW})- f^{\mathrm{s}}_{\ell_1}(\fW^\star)
        &=\frac{1}{m}\sum_{i\in\bar S}\left|\inner{A_i}{{\gamma}\Delta X_1+{\gamma^2}\Delta X_2}\right|+\frac{1}{m}\sum_{i\in S}\left(|\inner{A_i}{{\gamma}\Delta X_1+{\gamma^2}\Delta X_2}+\err_i|-|\err_i|\right)\\
		&\geq \frac{1}{m}\sum_{i\in\bar S}\left|\inner{A_i}{{\gamma}\Delta X_1+{\gamma^2}\Delta X_2}\right|-\frac{1}{m}\sum_{i\in S}|\inner{A_i}{{\gamma}\Delta X_1+{\gamma^2}\Delta X_2}|\\
		&\geq {\gamma}\left(\frac{1}{m}\sum_{i\in\bar S}\left|\inner{A_i}{\Delta X_1}\right|-\frac{1}{m}\sum_{i\in S}|\inner{A_i}{\Delta X_1}|\right)-\frac{{\gamma^2}}{m}\sum_{i=1}^m|\inner{A_i}{\Delta X_2}|,
	    \end{aligned}
	\end{equation*}
where the first and second inequalities follow from triangle inequality. Note that $\rank(\Delta X_1)\leq 2r$. Similarly, we have $\rank(\Delta X_2)\leq k$ and $\norm{\Delta X_2}_F\leq {1}$. Therefore, we have
	{\begin{align*}
	    \min_{\Delta\fW\in \cB}\left\{f^{\mathrm{s}}_{\ell_1}(\fW^\star{+\gamma\Delta\fW})- f^{\mathrm{s}}_{\ell_1}(\fW^\star)\right\}\geq &{\gamma}\cdot\underbrace{\min_{\rank(\Delta X_1)\leq 2r}\biggl\{\frac{1}{m}\sum_{i\in\bar S}\left|\inner{A_i}{\Delta X_1}\right|-\frac{1}{m}\sum_{i\in S}|\inner{A_i}{\Delta X_1}|\biggr\}}_{:=G_1}\\
     &-{\gamma^2}\underbrace{\max_{{\Delta X_2}\in \cS^{d_1\times d_2}_{k}}\biggl\{\frac{1}{m}\sum_{i=1}^m\left|\inner{A_i}{\Delta X_2}\right|\biggr\}}_{:=G_2},
	\end{align*}}
	Similar to the proof of \Cref{thm_benign_l1}, one can show that that $G_1> 0$ with probability at least $1-\exp(-\Omega(\max\{d_1,d_2\}r))$. On the other hand, applying $\ell_1/\ell_2$-RIP to $G_2$ with $\delta = \sqrt{\frac{c\max\{d_1,d_2\}k}{m}}$, we have
	\begin{equation*}
	    G_2\lesssim {\sqrt{\frac{2}{\pi}}+\sqrt{\frac{\max\{d_1,d_2\}k}{m}}},
	\end{equation*}
	with probability at least $1-\exp(-\Omega(\max\{d_1,d_2\}k))$. Combining the above inequalities leads to
	{\begin{equation*}
	    \begin{aligned}
	        &\min_{\fW^{\star}\in \cW(2r)}\min_{\Delta\fW\in \cB}\left\{f^{\mathrm{s}}_{\ell_1}(\fW^\star{+\gamma\Delta\fW})- f^{\mathrm{s}}_{\ell_1}(\fW^\star)\right\}
            \geq {\gamma}G_1-{\gamma^2}G_2
            \gtrsim  -\left(\sqrt{\frac{2}{\pi}}+\sqrt{\frac{\max\{d_1,d_2\}k}{m}}\right)\gamma^2,
	    \end{aligned}
	\end{equation*}}
	with probability at least $1-\exp(-\Omega(\max\{d_1,d_2\}r))$. This completes the proof. $\hfill\square$

 \noindent{\it Proof of \Cref{thm_critical_sensing_sym}.} The proof of the first statement of \Cref{thm_critical_sensing_sym} is identical to the proof of the first statement of \Cref{thm_benign_l1}. The proof of the second statement of \Cref{thm_critical_sensing_sym} follows directly from the proof of the second statement of \Cref{thm_benign_l1} after noting that, for the symmetric matrix sensing, every true solution is indeed rank-balanced.$\hfill\square$

{\section{Strict Saddle Nature}
\label{sec::strict-saddle}
We finally provide the proof of \Cref{cor::sym} by combining Theorems~\ref{thm_no_benign_sym} and~\ref{thm_critical_sensing_sym}. The proof for \Cref{cor::asym} follows identically; therefore, we omit it for brevity.

\noindent{\it Proof of \Cref{cor::sym}.}
First, we prove that all true solutions become critical points with high probability, provided that $m\gtrsim dr$. In light of Lemma~\ref{lem_critical}, it suffices to show that, with high probability, $(f^{\mathrm{s}}_{\ell_1})'(W^\star, \Delta W)\geq 0$ for every $W^\star\in \cW$ and every direction $\norm{\Delta W}_F=1$. According to the definition of the directional derivative, we need to show that 
\begin{equation}
    \lim_{\gamma\to 0^+}\frac{f^{\mathrm{s}}_{\ell_1}(W^\star+\gamma\Delta W)-f^{\mathrm{s}}_{\ell_1}(W^\star)}{\gamma}\geq 0, \qquad \forall W^\star\in \cW, \forall \Delta W: \norm{\Delta W}_F=1.
\end{equation}
The above inequality is readily implied by the first statement of \Cref{thm_critical_sensing_sym}, provided that $m\gtrsim \frac{dr}{(1-2p)^2}\asymp dr$. On the other hand, when $m\lesssim pp_0d(k-r)\asymp dk$, \Cref{thm_no_benign_sym} can be invoked to show that, for every $W^\star\in \cW$, there exists a direction $\norm{\Delta W}_F=1$, and scalars $c>0$ and $\bar\gamma>0$ such that, for every $0\leq \gamma\leq \bar\gamma$, we have
\begin{equation}
    f^{\mathrm{s}}_{\ell_1}(W^\star+\gamma\Delta W)-f^{\mathrm{s}}_{\ell_1}(W^\star)\leq -c\gamma^2.
\end{equation}
By combining the above two statements, we conclude that, when $dr\lesssim m\lesssim dk$, all true solutions $W^\star\in\cW$ become critical and satisfy the definition of strict saddle points (\Cref{def_ssp}), with a probability stated in \Cref{cor::sym}. This completes the proof.$\hfill\square$

}
\section{Discussion and Future Directions}\label{sec::discussion}
The current paper studies the optimization landscape of the robust low-rank matrix recovery. In a departure from the noiseless settings, our results reveal that true solutions do not invariably correspond to local or global optima. Instead, we offer a full characterization of the local geometry of the true solutions for both asymmetric/symmetric matrix sensing and completion in terms of sample size, rank-balancedness, and coherence. Our key finding is that, under moderate conditions on the model and sample size, the true solutions may emerge---not as local optima---but as strict saddle points. 

When the loss is smooth, recent results have shown that different first-order algorithms, including randomly initialized gradient descent~\cite{lee2016gradient}, stochastic gradient descent~\cite{fang2019sharp}, and perturbed gradient descent~\cite{jin2017escape}, exhibit a behavior known as saddle-escaping. Specifically, they converge to second-order stationary points, either with a high probability or almost surely. 

However, in the case of nonsmooth $\ell_1$-loss, first-order algorithms are not equipped with a guarantee of escaping from strict saddle points~\cite{davis2022proximal}. In fact, our observations have revealed quite the opposite: the sub-gradient method consistently converges to the true solution, as substantiated either through formal proof~\cite[Theorem 1]{ma2023global} or empirical evidence, despite these points emerging as strict saddle points. 

To support this observation empirically, we present evidence from both symmetric and asymmetric matrix sensing scenarios. First, we consider an instance of the symmetric matrix sensing where the ground truth $\trueX$ has dimension $d=20$ and rank $r=3$. Our considered model is fully over-parameterized, with $k=d=20$ and a sample size of $m=90$, which equals $(3/2)\cdot dr$. Notably, approximately $10\%$ of the measurements suffer from gross corruption due to outlier noise. Note that, in this regime, the global minima achieve zero loss, and hence cannot correspond to the true solution. Given the SVD of the ground-truth $\trueX = V^\star \Sigma^\star {V^\star}^\top$, we consider a true solution $\trueW = V^\star {\Sigma^\star}^{1/2}$. For each experimental trial, we initialize the sub-gradient method at $\trueW+E$, where the elements of $E$ are drawn from a Gaussian distribution with a mean of zero and a variance of $5\times 10^{-5}$. For the asymmetric matrix sensing, we maintain a similar setup, albeit with a variation in the sample size, which is now set to $m=170 = (17/6)\max\{d_1,d_2\}r$. 

Figure~\ref{fig::non-active} illustrates the distances between the iterations of the algorithm and the ground truth in both scenarios. It becomes evident that in a substantial number of experiments, the sub-gradient method with an exponentially decaying stepsize converges linearly to a true solution. We put forth two possible explanations for this behavior of the sub-gradient method:

{\paragraph{True solutions emerge as non-active strict saddle points}
A recent study by \cite{davis2022proximal} introduced the concept of active strict saddle points to identify critical points that can be escaped using first-order algorithms. In essence, a critical point is considered an active strict saddle if it lies on a manifold 
$\cM$ such that: (i) the loss exhibits sharp variation outside $\cM$; (ii) the loss, when restricted to $\cM$, is smooth; and (iii) the Riemannian Hessian of the loss on $\cM$ has at least one negative eigenvalue (see \cite[Figures 1 and 2]{davis2022proximal}). A wide range of first-order methods, including randomly initialized proximal methods \cite[Theorem 5.6]{davis2022proximal} and perturbed sub-gradient methods \cite[Theorem 3.1]{davis2022escaping}, escape active strict saddle points almost surely. However, these methods are not guaranteed to escape non-active strict saddle points.

Given this, one possible explanation for the convergence of the sub-gradient method is that the true solutions may manifest as non-active strict saddle points.

\paragraph{Sub-gradient method with random initialization is unable to escape strict saddle points} \cite{bianchi2023stochastic} conjectured that, unlike proximal methods, mere randomness in the initial point may be insufficient for the sub-gradient method to escape saddle points in robust low-rank matrix recovery, even when these points manifest as active strict saddle points (see also~\cite{davis2021subgradient} for a more detailed discussion). Consequently, another possible explanation for the convergence of the sub-gradient method could be its inherent lack of a saddle-escaping property.

Indeed, we view proving or disproving the above conjectures as a compelling challenge for future research.}

{Finally, we would like to emphasize that a complete theoretical explanation of the observed empirical behavior of the sub-gradient method around strict saddle points could provide valuable insights into the benefits of early stopping---i.e., halting the algorithm before full convergence to a critical point---which is widely used in training modern over-parameterized machine learning models. The emergence of the true solution as an unstable (e.g., non-active) strict saddle point could be one of the explanations for why early stopping is necessary in practice. Indeed, if the algorithm is inevitably bound to escape the true solution due to its instability stemming from its strict saddle nature, then early stopping would be essential in preventing such an escape.}

\begin{figure*}
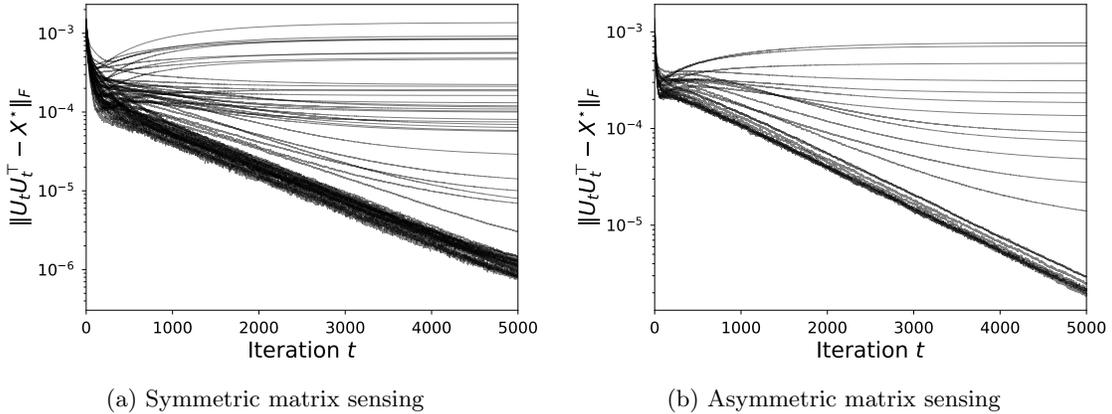

    % \vskip 0.2in
    \begin{center}
        \subfloat[Symmetric matrix sensing]{
            {\includegraphics[width=0.45\textwidth]{figure/non_active_sym.pdf}}\label{fig::non-active_sym}}
        \subfloat[Asymmetric matrix sensing]{
            {\includegraphics[width=0.45\textwidth]{figure/non_active_asym.pdf}}\label{fig::non-active_asym}}
    \end{center}
    \caption{\footnotesize We apply the sub-gradient method with an exponentially decaying stepsize to noisy instances of symmetric and asymmetric matrix sensing with $\ell_1$-loss. When the algorithm is initialized randomly but close enough to the true solution, a significant portion of the trajectories converge to the true solution.}\label{fig::non-active}
\end{figure*}

\section*{Acknowledgements}
 This research is supported, in part, by NSF Award DMS-2152776, ONR Award N00014-22-1-2127.
	
	\bibliographystyle{alpha}
	\bibliography{ref.bib}

\end{document}